\documentclass[twoside]{article}
\usepackage{enumerate}%for enumeration
\usepackage{amsfonts}%for calligraphic R
\usepackage{algorithm}
\usepackage{algorithmic}
\usepackage{graphicx}%for graphics
\usepackage{mathtools}
\usepackage{float}
\usepackage{amsthm}
\usepackage{amssymb}
\usepackage{wrapfig}
\usepackage{dsfont}
\usepackage{subfig}
\usepackage[colon]{natbib}
\usepackage{bbm}
\usepackage{bm}
\usepackage[accepted]{aistats2018}
\usepackage{environ}

\usepackage{hyperref}
\usepackage[capitalize]{cleveref}

\newtheorem{lemma}{Lemma}

\newtheorem{theorem}{Theorem}
\newtheorem*{theorem*}{Theoremm}

\newtheorem{prop}{Proposition}
\newtheorem{corr}{Corollary}
\newtheorem{assumption}{Assumption}

\newsavebox\ideabox
\newenvironment{idea}
  {\begin{equation}
   \begin{lrbox}{\ideabox}
   \begin{minipage}{\dimexpr\columnwidth-2\leftmargini}
   \setlength{\leftmargini}{0pt}%
   \begin{quote}}
  {\end{quote}
   \end{minipage}
   \end{lrbox}\makebox[0pt]{\usebox{\ideabox}}
   \end{equation}}
\mathchardef\mhyphen="2D

\renewcommand{\P}{\mathbb{P}}
\newcommand{\lucbrank}{{\tt LUCBRank}}

\newcommand{\klucb}{{\tt KL\mhyphen UCB}}
\newcommand{\chernoff}[1]{d^\ast(#1)}
\newcommand{\KL}[1]{\text{KL}(#1)}
\newcommand{\hstar}{H^\ast_{\epsilon,b}}

\DeclareMathOperator*{\argmax}{arg\,max\,}
\DeclareMathOperator*{\argmin}{arg\,min\,}

\begin{document}



\twocolumn[
\aistatstitle{Adaptive Sampling for Coarse Ranking}
%\aistatsauthor{ Author 1 \And Author 2 \And  Author 3 }
\vspace{-0.2in}
\aistatsauthor{Sumeet Katariya \And Lalit Jain \And Nandana Sengupta}
\aistatsaddress{katariya@wisc.edu \And
    lalitj@umich.edu \And
nandana@uchicago.edu}
%\aistatsaddress{University of Wisconsin-Madison \\ \emph{katariya@wisc.edu} \And
%    University of Michigan Ann-Arbor \\ \emph{lalitj@umich.edu} \And
%University of Chicago \\ \emph{nandana@uchicago.edu}}
\vspace{-0.2in}
\aistatsauthor{James Evans \And Robert Nowak}
\aistatsaddress{jevans@uchicago.edu \And
rdnowak@wisc.edu}
%\aistatsaddress{University of Chicago\\ \emph{jevans@uchicago.edu} \And
%University of Wisconsin-Madison\\ \emph{rdnowak@wisc.edu}}
\vspace{-0.1in}
\runningauthor{Katariya, Jain, Sengupta, Evans, Nowak}
]

\begin{abstract}
  We consider the problem of active \emph{coarse ranking}, where the
  goal is to sort items according to their means into clusters of
  pre-specified sizes, by adaptively sampling from their reward
  distributions.  This setting is useful in many social science
  applications involving human raters and the \emph{approximate} rank
  of \emph{every} item is desired. Approximate or coarse ranking can
  significantly reduce the number of ratings required in comparison to
  the number needed to find an exact ranking. We propose a
  computationally efficient PAC algorithm $\lucbrank$ for coarse
  ranking, and derive an upper bound on its sample complexity. We also
  derive a nearly matching distribution-dependent lower bound.
  Experiments on synthetic as well as real-world data show that
  $\lucbrank$ performs better than state-of-the-art baseline methods,
  even when these methods have the advantage of knowing the underlying
  parametric model.
\end{abstract}

%!TEX root = Paper.tex
\section{Introduction}
\label{sec:introduction}
We consider the problem of efficiently sorting items according to their
means into clusters of pre-specified sizes, which we refer to as
\emph{coarse ranking}. In many big-data applications, finding the
\emph{total} ranking can be infeasible and/or unnecessary, and we may
only be interested in the top items, bottom items, or quantiles.
Consider for instance the problem of assessing the safety of
neighborhoods from pairwise comparisons of Google street view images,
as is done in the Place Pulse project \citep{naik2014streetscore},
which can be applied to develop social policy \citep{dubey2016deep}.
Finding a complete ordering of the images in this case is impractical
because many images are difficult to compare i.e., their safety scores
are very close (see \cref{sec:chicagostreetview}).  Furthermore, a
total ordering may be unnecessary from a public policy point of view,
since the \emph{approximate} rank of every image on the safe-unsafe
spectrum may suffice.
%REPLACE WITH As another example, consider the problem of assigning grades to students in
%massive open online courses using peer reviews \cite{shah2013case}. Here
%again, only the \emph{quantile} that every student belongs to is desired (the
%number of quantiles is determined by the grading scale). A total ordering is
%unnecessary and may also be infeasible if the number of students is large. The
%recommender-systems solution of adaptively finding the top items does not help in
%this case because the grade of \emph{every} student is desired, not just those
%at the top.

Motivated by these applications, we model the coarse ranking problem as
follows. Given $K$ random variables, $c \ge 2$ clusters, and cluster boundaries
$1 \le \kappa_1 < \kappa_2 < \dots < \kappa_{c-1} < \kappa_c = K$, the goal is
to reliably identify the $\kappa_1$ random variables with the highest means, the
$\kappa_2-\kappa_1$ random variables with the highest means among the remaining
$K-\kappa_1$ random variables, and so on, by observing samples from their
reward distributions (for a precise formulation see \cref{sec:setting}). The focus of
this paper is on algorithms that achieve this clustering by requesting samples adaptively.
The coarse ranking setting applies to the scenarios above, and also
subsumes many well-studied problems.  The problem of finding the best item
corresponds to $\kappa_1=1, \kappa_2=K$. The problem of finding the top-$m$
items corresponds to $\kappa_1=m, \kappa_2=K$.  The problem of sorting the items
into $c$ equal-sized clusters corresponds to $\kappa_i = \text{round}(iK / c),
1\le i \le c$.  Finally, the complete ranking can be obtained by setting
$\kappa_i=i, 1\le i\le i \le K$. 

The problem of \emph{completely sorting} items is in general hard in real-world
applications, and does not exhibit gains from adaptivity.
\citet{maystre2017just} who analyze the performance of Quicksort, observe in
their real-world experiments:
\begin{idea}
   ``The improvement is noticeable but modest. We notice that item parameters are
    close to each other on average; $\dots$ This is because there is a considerable fraction
    of items that have their parameters (means) very close to one another $\dots$
    Figuring out the exact order of these images is therefore difficult and
    probably of marginal value.''
    \label{maystrequote}
\end{idea}
The fact that \emph{adaptivity doesn't help for complete ranking} is true not just for
Quicksort, but other adaptive algorithms as well - as we observe in our
experiments. Adaptivity does however help for coarse ranking, and this can be explained.
Consider the case when the $K$ items have bounded reward distributions, and their means are equally separated, with a gap
$\Delta$ between consecutive means.  Correctly ordering any two consecutive
items requires $\Omega(1/\Delta^2)$ samples, and thus \emph{any} algorithm would
require $\Omega(K/\Delta^2)$ to find a total ordering.  A non-adaptive algorithm
sampling the items uniformly would gather approximately equal samples from every
item, and hence will find the correct ranking after roughly these many samples
(up to perhaps log factors). Thus adaptivity doesn't help in this case.
However, if the goal is to find only the quartiles say, an adaptive algorithm
can quickly stop sampling items that are far from the quartile boundaries and
gain over non-adaptive algorithms. %We verify the futility of adaptive methods
%for finding the total ordering on real-world datasets in
%\cref{sec:chicagostreetview}.

In this work, we make six contributions. First, we motivate the coarse
ranking setting. We do this by arguing that most real-life problems
have high noise, and by explaining why adaptive methods are ineffective
in producing a complete ranking in these high-noise regimes
(\cref{sec:motivation}). Second, we precisely formulate the online
probably approximately correct (PAC)-coarse ranking problem with error
tolerance $\epsilon$ and failure probability $\delta$ that can model
real-valued as well as pairwise comparison feedback
(\cref{sec:setting}). Third, we propose a nonparametric PAC Upper
Confidence Bound (UCB)-type algorithm $\lucbrank$ to solve this
problem. To the best of our knowledge, this is the first UCB-type
algorithm for ranking (\cref{sec:algorithm}). Fourth, we analyze the
sample complexity of $\lucbrank$ and prove an upper bound which is
inversely proportional to the distance of the item to its closest
cluster boundary, where the distance is measured in terms of Chernoff
information (\cref{sec:analysis}). Fifth, we also prove a nearly
matching distribution-dependent lower bound. The contribution of an
item to the lower bound is inversely proportional to the distance of
the item to the closest item in an adjacent cluster, with distance in
this case measured using KL-divergences (\cref{sec:lower bound}).
Finally, we compare the performance of our algorithm to several
baselines on synthetic as well as real-world data gathered using MTurk,
and observe that it performs $2\hbox{ - }3$x better than existing
algorithms even when they have the advantage of knowing the underlying
parametric model (\cref{sec:experiments}).

\subsection{Ranking using Pairwise Comparisons}
\label{sec:ranking pc}
We use the term direct-feedback or real-rewards to indicate a
setting where the learner can sample directly from the item's reward distribution. Our
algorithm is stated for this setting. In contrast, in the pairwise-comparison or
dueling setting, the learner compares two items and receives $1$-bit feedback
about who won the duel. We next explain how to translate our algorithm to this setting.

Any algorithm designed to solve the direct-feedback coarse ranking problem can also be used with pairwise comparison feedback using Borda
reduction \citep{jamieson2015sparse}. According to this technique, whenever the
algorithm asks to draw a sample from item $i$, we compare item $i$ to a randomly
chosen item $j$, and ascribe a reward of $1$ to item $i$ if $i$ wins the duel,
and $0$ otherwise. This is equivalent to the rewards being sampled from a
Bernoulli distribution with means given by the Borda scores of the items. The
Borda score of an item $i$ is defined as 
\begin{equation}
\begin{small}
    p_i := \frac{1}{K-1} \sum\limits_{j \ne i}^{} \P(i > j).
    \label{eq:bordadefn}
    \vspace{-10pt}
\end{small}
\end{equation}

%!TEX root = Paper.tex

\section{Related Work}
\label{sec:relatedwork}
There is extensive work on ranking from noisy pairwise comparisons, we refer the
reader to excellent surveys by \citet{busa2014survey,
agarwal2016ranking}. We discuss the most relevant work next.

\subsection{Ranking from Pairwise Comparisons}
The pairwise comparison matrix $P$ (where $P_{ij} = \mathbb{P}(i>j)$) and
assumptions on it play a major role in the design of ranking algorithms
\citep{agarwal2016ranking}. A sequence of progressively relaxed assumptions on
$P$ can be shown where ranking methods that work under restrictive assumptions
fail when these assumptions are relaxed \citep{rajkumar2014statistical,
rajkumar2015ranking}.  Spectral ranking algorithms have been proposed when
comparisons are available for a fixed set of pairs \citep{negahban2012iterative,
negahban2012rank}; this corresponds to a partially observed $P$ matrix.
\citet{braverman2009sorting, wauthier2013efficient} propose and analyze
algorithms for the noisy-permutation model; this corresponds to a $P$ matrix
which has two types of entries: $1-p$ in the upper triangle and $p$ in the lower
triangle (assuming the true ordering of the items is $1\dots K$). They also
focus on settings where queries cannot be repeated. Our work makes no
assumptions on the $P$ matrix and ranks items using their Borda scores.  This is
important given the futility of parametric models to model real-life scenarios
\citep{shah2016stochastically}. 

Quicksort is another highly recommended algorithm for ranking using noisy
pairwise comparisons. \citet{maystre2017just} study Quicksort
under the BTL noise model, and \citet{alonso2003sorting} analyze Quicksort under
the noisy permutation model. We comment on these in
\cref{sec:motivation}.

\citet{jamieson2011active} propose an algorithm for active ranking from pairwise
comparisons when points can be embedded in Euclidean space.
\citet{ailon2012active} consider ranking when query responses are fixed. More
recently, \citet{agarwal2017learning} consider top-$m$ item identification and ranking
under limited rounds of adaptivity, \citet{falahatgar2017maximum} consider the
problem of finding the maximum and ranking assuming strong-stochastic
transitivity and the stochastic-triangle inequality. We do not need these assumptions.

Our setting is closest to the setting proposed by \citet{heckel2016active}, in
the context of ranking using pairwise comparisons.  Our setting however applies
to real-valued rewards as well as pairwise comparison feedback.  Furthermore,
our setting incorporates the notion of $\epsilon$-optimality which allows the
user to specify an error tolerance \citep{even2006action}. This is important in practice if the item means
are very close to each other. Finally, as they note, their Active Ranking (AR) algorithm is an 
elimination-style algorithm, our $\lucbrank$ is UCB-style; it is known that the
latter perform better in practice \citep{jiang2017practical}. 
We also verify this empirically in \cref{sec:chicagostreetview}, and observe that $\lucbrank$ requires $2$-$3$x fewer samples than AR in our synthetic as well as real-world experiments (see \cref{fig:expt1} and \cref{fig:chicagoresults}).

\subsection{Relation to Bandits}
The idea of sampling items based on lower and upper confidence bounds is
well-known in the bandits literature \citep{auer2002using}. However, these
algorithms either focus on finding the best or top-$m$ items
\citep{audibert2010best, kalyanakrishnan2012pac,kaufmann2015complexity,
chen2017nearly}, or on
minimizing regret \citep{bubeck2012regret}.  This is the first work to our
knowledge that employs this tool for ranking.

%!TEX root = Paper.tex
\vspace{-5pt}
\section{Motivation}
\label{sec:motivation}
\vspace{-5pt}
We argue that existing adaptive methods offer no significant gains over their non-adaptive counterparts when the goal is to find a complete ranking, and coarse ranking is more appropriate for many real-world applications. We provide brief theoretical justification for this claim in the discussion after quote~(\ref{maystrequote}), and empirically verify this behavior in \cref{fig:completerankingisfutile}. In this section, we focus on Quicksort, because it has been well-studied under multiple noise models. Quicksort has optimal sample complexity when
comparisons are noiseless \citep{sedgewick2011algorithms} and is naturally 
appealing when comparisons are noisy \citep{maystre2017just}.
Intuitively it feels like the right thing to do - by comparing an item with the
pivot and putting it left or right appropriately, Quicksort performs a binary
search for the true position of an item. However it is far from optimal under
two noise models as we argue next. 

First, consider the noisy-permutation (NP) noise model
\citep{feige1994computing} where the outcomes of pairwise comparisons are
independently flipped with an error probability $p$. In the first stage of
Quicksort, every item that is compared with the pivot and put in the wrong
bucket contributes on average $\frac{K}{2}$ to the Kendall tau error (total
number of inverted pairs). Now, $Kp$ items are put in the wrong bucket on
average in the first stage of Quicksort, and hence the total number of
inverted pairs is at least $\Omega(K^2p)$. \citet{alonso2003sorting} show that $\Theta(K^2p)$ is indeed the expected number of inversions. This is far from optimal because \citet{braverman2009sorting} propose
an algorithm which has a Kendall tau error of $O(K)$ \emph{with high probability}, using $K \log K$ comparisons (same as quicksort).
\citet{alonso2003sorting} conjecture that for quicksort to have $O(K)$ expected inversions, $p$ needs to go down faster than $1/K$, like $\frac{1}{K
\log K}$. As the above calculation shows, they conjecture that this is because Quicksort is extremely brittle: ``the main contribution (to the total inversions)
comes from the `first' error, in some sense.''
One may be able to get rid of this lack of robustness by repeating queries, but this requires knowledge of the error probability $p$ or adapting to its unknown value. This is possible, but as we argue shortly, a good model for real-world problems where comparisons are made by humans is one where $p$ increases to $1/2$ as $K$ grows, since it becomes more difficult to compare adjacent items in the true ranking as $K$ increases. Quicksort certainly fails in this regime.

The other class of well-studied noise models are the Bradley-Terry-Luce (BTL) \citep{bradley1952rank} or Thurstone \citep{thurstone1927law} models, which assume a $K$-dimensional weight vector that measures the quality of each item, and the pairwise comparison probabilities are determined via some fixed function of the qualities of pair of objects. These models are more realistic than the NP model since under these models, comparisons between items that are far apart in the true ranking are less noisy than those between nearby items. \citet{maystre2017just} analyze the expected number of inversions of Quicksort under the BTL model, and show that when the average gap between adjacent items is $\Delta$, the expected number of inversions is $O(\Delta^{-3})$. They note however that real-world datasets have extremely small $\Delta$ ($\hat{\Delta}^{-1}=376$ in their experiments) and Quicksort performs no better than random (see quote~(\ref{maystrequote})). We make similar observations about the inefficacy of Quicksort (and other adaptive algorithms) in our real-world experiments (see \cref{fig:completerankingisfutile}).

% We argue next that a constant gap $\Delta$ between consecutive item means is not a good model for experiments where humans rate items, and that the gap should decrease with $K$. Let $m = \lceil \frac{1}{\Delta} \log (19) \rceil$. Then, assuming the logistic model, the $m$-th item beats the $1^\text{st}$ item with probability $\ge 0.95$, the $2m$-th item beats the $m$-th item with probability $\ge 0.95$, and so on. Thus, items that are $m$-apart can be considered distinguishable. In a task like Place Pulse where humans rate street view images according to their perceived safety \citep{naik2014streetscore}, or the task in \citet{wood2017towards} where humans rate face images according to the strength of their emotions, there are only a finite number of levels (safety levels, emotional levels) that humans can distinguish. As more items are added, the number of adjacent items that are indistinguishable increases, and this is not captured by the constant-gap model where $m$ is independent of $K$.

The problem in finding an exact/total ranking is that if the means of the items lie in a bounded range, e.g., $[0,1]$, then the minimum gap must decrease at least linearly with $K$ and many items become essentially indistiguishable.  To see this, suppose there is a constant gap $\Delta$ between consecutive means and let $m = \lceil \frac{3}{\Delta} \rceil$. Then, assuming the logistic model, the $m$-th item beats the $1^\text{st}$ item with probability $\ge 0.95$, the $2m$-th item beats the $m$-th item with probability $\ge 0.95$, and so on. Thus, items that are $m$-apart can be considered distinguishable.  Assuming the range of possible means is bounded implies that $\Delta \propto 1/K$.  Thus, the number of items that are essentially indistiguishable increases linearly with $K$, suggesting that seeking a total ranking is a futile effort.  This situation arises in applications such as Place Pulse where humans rate street view images according to their perceived safety \citep{naik2014streetscore}, or the task in \citet{wood2017towards} where humans rate face images according to the strength of their emotions.

Coarse ranking allows the experimenter to set the number of clusters in accordance with the number of distinguishable levels, and thus frees the algorithm from the task of distinguishing incomparable items. In this sense, it converts a high-noise problem to a low-noise one. Even though the gap between adjacent items is small, most items are far from their nearest cluster boundary, and an adaptive algorithm can stop sampling these items early.

%!TEX root = Paper.tex

\vspace{-5pt}
\section{Setting}
\label{sec:setting}
\vspace{-5pt}
In this section, we precisely formulate the coarse ranking setting. For ease of reference, we use terminology from the bandits literature and refer to an item as arm. Also, pulling or drawing an arm is equivalent to sampling from the item's reward distribution.

Consider a multi-armed bandit with $K$ arms. Each arm $a$ corresponds to a Bernoulli distribution with an unknown mean $p_a$, denoted $\mathcal{B}(p_a)$. A draw / pull of arm $a$ yields a reward from distribution $\mathcal{B}(p_a)$. Without loss of generality, assume the arms are numbered so that $p_1 \ge p_2 \dots \ge p_K$.

Given an integer $c \ge 2$ representing the number of clusters, let $1\le \kappa_1 < \kappa_2 < \dots < \kappa_c=K$ be a collection of positive integers. Any such collection of positive integers defines a partition of $[K]$ into $c$ disjoint sets of the form 

\vspace{-18pt}
{\small
\begin{align}
    &M^\ast_1 := \{1,\dots,\kappa_1\},\; M^\ast_2 := \{\kappa_1+1,\dots,\kappa_2\},
    \dots, \nonumber\\
    &\dots,\; M^\ast_c := \{\kappa_{c-1}+1,\dots,K\}.
    \label{eq:trueclusters}
\vspace{-8pt}
\end{align}
}%
To solve the \emph{coarse ranking} problem given a set of cluster boundaries $(\kappa_i)_{i=1}^c$, an algorithm may sample arms of the $K$-armed bandit and record the results; the algorithm is required to terminate and cluster the arms into an ordered set of disjoint sets of the form (\ref{eq:trueclusters}). We refer to this output as a \emph{coarse ranking}.

We next define the notion of $\epsilon$-tolerance. For some fixed tolerance $\epsilon \in [0,1]$ and $1 \le i \le c$, let $M^\ast_{i,\epsilon}$ be the set of all arms that should be in cluster $i$ upto a tolerance $\epsilon$, i.e.
\[
\begin{small}
    M^\ast_{i,\epsilon} := \{a: p_{\kappa_{i-1}+1}+ \epsilon \ge p_a \ge
    p_{\kappa_i}-\epsilon\},
\end{small}
\]
(with the convention that $p_{\kappa_0}=1$).  Note that the true set of arms in
cluster $i$: $M^\ast_i:= \{\kappa_{i-1}+1,\dots,\kappa_i\}$, is a subset of
$M^\ast_{i,\epsilon}$; the latter set contains in addition arms that are
$\epsilon$ close to the boundary. 

For a given mistake probability $\delta \in [0,1]$ and a given error tolerance
$\epsilon \in [0,1]$, we call an algorithm
$(\epsilon, \delta)$-PAC if, with a probability greater that $1-\delta$, after using a
finite number of samples, it returns a rank for each arm such that the
$i^\text{th}$ ranked cluster according to the returned ranking is a subset of
$M^\ast_{i,\epsilon}$ for all $1 \le i \le c$. Formally, if $\sigma(a)$ is the
rank of arm $a$ returned by the algorithm after using a finite number of
samples, we can define the empirical cluster $i$ as
\[
\begin{small}
    \hat{M}_i := \{a: \kappa_{i-1}+1 \le \sigma(a) \le \kappa_i\},
\end{small}
\]
and we say the algorithm is $(\epsilon, \delta)$-PAC if
\begin{equation}
\begin{small}
    \P \left( \exists \,i \text{ such that }\hat{M}_i \not\subseteq
    M^\ast_{i,\epsilon}\right)
    \le \delta.
\end{small}
\label{eq:meaningofdeltapac}
\end{equation}

%!TEX root = Paper.tex
%
\vspace{-20pt}
\section{Algorithm}
\label{sec:algorithm}
\vspace{-5pt}
Let $(\kappa_1,\dots,\kappa_c=K)$ be the cluster boundaries. We
describe here the $\lucbrank$ algorithm using generic confidence intervals
$\mathcal{I}_a = [L_a(t), U_a(t)]$, where $t$ indexes rounds of the algorithm. Let
$N_a(t)$ be the number of times arm $a$ has been sampled up to round $t$, and $S_a(t)$ be the sum of rewards
of arm $a$ up to round $t$. Let $\hat{p}_a(t) = \frac{S_a(t)}{N_a(t)}$ be the
corresponding empirical mean reward. Sort the arms in the decreasing order of
their empirical mean rewards, and for $1 \le i \le c-1$, let $J_i(t)$ denote the
$\kappa_i$ arms with the highest empirical mean rewards. Define 
\begin{align}
    l_t^i := \argmin\limits_{a\in J_i(t)} L_a(t), \qquad u_t^i := \argmax\limits_{a \notin J_{i}(t)} U_a(t)
    \label{eq:criticalarmsatboundaryi}
\end{align}
to be the two \emph{critical} arms from $J_i(t)$ and $J_i^c(t)$ that are likely to be
misclassified (see \cref{fig:algvisual}). 
\begin{figure}[h]
    \centering
    \includegraphics[width=.6\columnwidth]{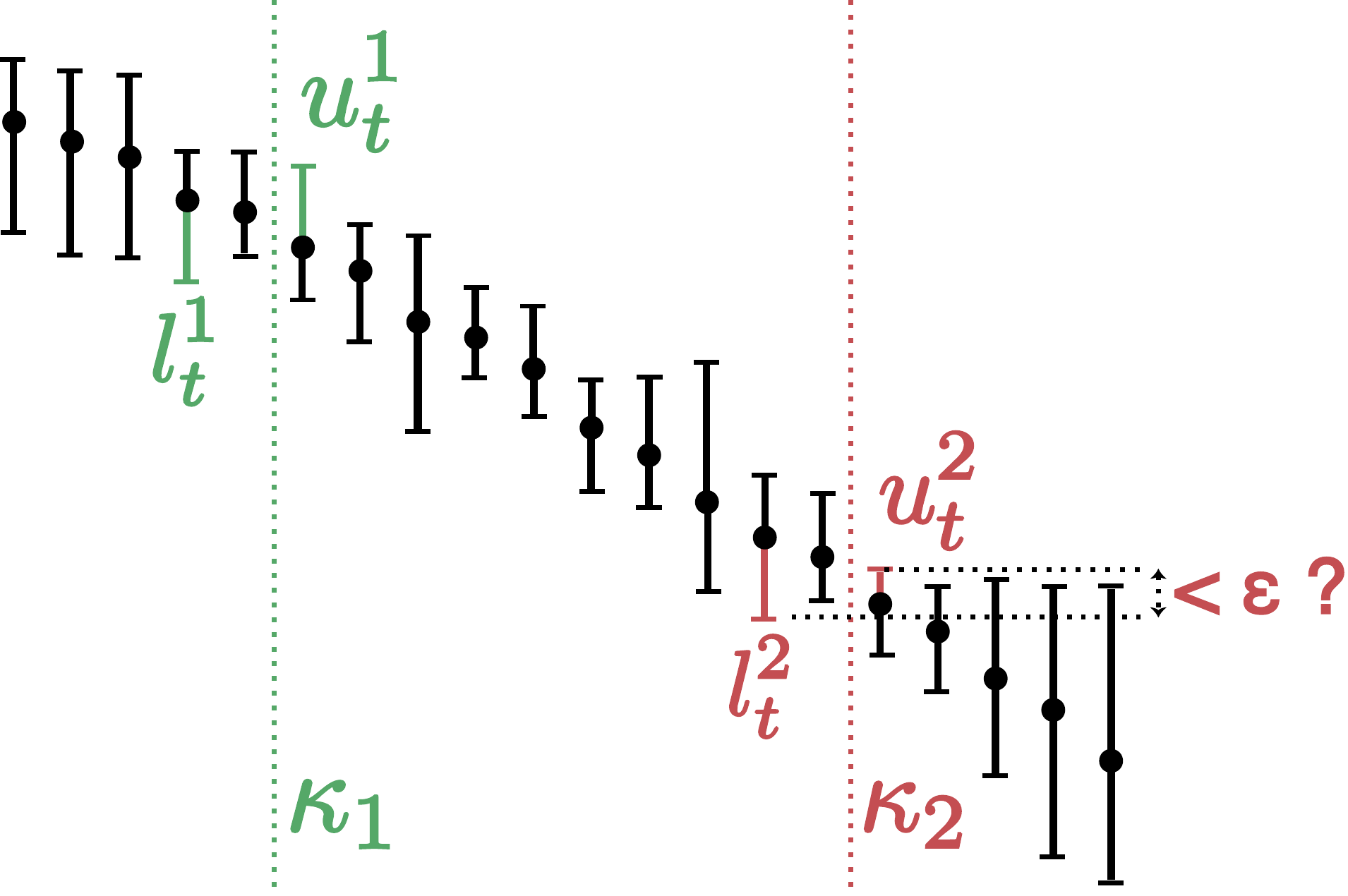}
    \caption{A visualization of $\lucbrank$ on a bandit instance with $K=20$
        arms, $c=3$ clusters, with boundaries at $\kappa_1=5, \kappa_2=15$. Also
        shown are the critical arms $l^i_t, u^i_t$ pulled at each boundary.
        The algorithm stops sampling a boundary when the confidence interval overlap
        is less than $\epsilon$.}
    \label{fig:algvisual}
    \vspace{-10pt}
\end{figure}
\begin{algorithm}[t]
	\caption{$\lucbrank$}
	\label{alg:main}
	\begin{algorithmic}[1]
        \STATE \textbf{Input:} $\epsilon > 0$, cluster boundaries $1\le \kappa_1,
        \dots, \kappa_c=K$
        \STATE $t \gets 1$
        \STATE $C \gets \{1,\dots,c-1\}$  \hfill //active cluster boundaries
        \STATE
        \FOR{$a = 1,\dots,K$}
        \STATE Sample item $a$, compute $U_a(1)$ and $L_a(1)$
        \ENDFOR
        \STATE
        \WHILE{$C \neq \varnothing$}
            \STATE // Sample active cluster boundaries
            \FOR{$i \in C$}
            \STATE Sample item $l_t^i$ 
            \STATE Sample item $u_t^i$
            \STATE (If pairwise comparing, compare item $l_t^i$ to a random other item, and compare item $u_t^i$ to a random other item. See \cref{sec:ranking pc})
            \ENDFOR
            \STATE $t = t+1$
            \STATE $\forall \,a \in [K]: $ Update reward-estimate $\hat{p}_a(t)$, number of samples $N_a(t)$, and confidence bounds $U_a(t), L_a(t)$ (see \eqref{eq:klupperbound})
            \STATE $\forall\, i\in C$: Compute $l_t^i, u_t^i$ (see \eqref{eq:criticalarmsatboundaryi})
            \STATE
            \STATE // Eliminate unambiguous cluster boundaries
            \FOR{$i \in C$}
            \IF{$U_{u_t^i}(t)-L_{l_t^i}(t) < \epsilon$}
                    \STATE $C = C \setminus i$
                \ENDIF
            \ENDFOR
        \ENDWHILE
        \STATE
        \STATE Return items sorted by their empirical mean rewards.
    \end{algorithmic}
\end{algorithm}

\cref{alg:main} contains the pseudocode of $\lucbrank$, which is also depicted
in \cref{fig:algvisual}. The algorithm maintains active cluster boundaries in
the set $C$, where a cluster boundary $i$ is active if the overlap of confidence
intervals in $J_i$ and $J_i^c$ is not less than $\epsilon$. In every round, it
samples both the critical arms at every active cluster boundary (lines
$11$-$15$). At the end of every round, it checks if the critical arms at
any boundary are separated according to the tolerance criterion, and removes
such boundaries from the active set (lines $21$-$25$).
For our experiments, we use $\klucb$ \citep{garivier2011kl} confidence intervals. For an exploration
rate $\beta(t,\delta)$, the $\klucb$ upper and lower confidence bounds for arm
$a$ are calculated as

\vspace{-15pt}
{\small
\begin{align}
    \begin{aligned}
    U_a(t) &:= \max\{q \in [\hat{p}_a(t), 1]: N_a(t) d(\hat{p}_a(t), q) \le
    \beta(t,\delta)\}, \\
    L_a(t) &:= \min\{q \in [0, \hat{p}_a(t)]: N_a(t) d(\hat{p}_a(t), q) \le
    \beta(t,\delta)\}. 
    \end{aligned}
\vspace{-40pt}
    \label{eq:klupperbound}
\end{align}
}%
where $d(x,y)$ is the Kullback-Leibler divergence between two Bernoulli distributions, given by $d(x,y) = x \log \tfrac{x}{y} + (1-x)\log
\tfrac{1-x}{1-y}$.

$\lucbrank$ can also be easily modified for pairwise-comparison queries:
whenever the algorithm calls for drawing an arm $i$, duel arm $i$ with another
arm chosen uniformly at random.

%!TEX root = Paper.tex

\vspace{-6pt}
\section{Analysis}
\label{sec:analysis}
\vspace{-4pt}
We prove the accuracy of $\lucbrank$ in \cref{thm:pacguarantee}, and give an
upper bound on the sample complexity in \cref{thm:samplecomplexity}. Our
distribution-dependent lower bound for the sample complexity of any $\delta$-PAC
algorithm is stated in \cref{thm:lowerbound}. All proofs can be found in the
Appendix. Recall that $1 \le \kappa_1 < \kappa_2 < \dots < \kappa_{c-1} <
\kappa_c=K$ are the cluster boundaries.
%\vspace{-1em}
\vspace{-3pt}
\subsection{PAC Guarantee}
\label{sec:pacguarantee}
\vspace{-3pt}
\cref{thm:pacguarantee} gives choices of $\beta(t,\delta)$ such that $\lucbrank$
is correct with probability at least $\delta$, in the sense defined by
\eqref{eq:meaningofdeltapac}.
\begin{theorem}
    $\lucbrank$ using $\beta(t,\delta) = \log(\tfrac{k_1 K t^\alpha}{\delta}) +
    \log \log (\tfrac{k_1 K t^\alpha}{\delta})$ with $\alpha > 1$ and $k_1 >
    \left( \frac{c-1}{2} \right)^\alpha +
    \tfrac{2e}{\alpha-1} + \tfrac{4e}{(\alpha-1)^2}$, is correct with
    probability $1-\delta$.
    \label{thm:pacguarantee}
\end{theorem}
\vspace{-4pt}
\subsection{Sample Complexity}
\vspace{-4pt}
Our sample complexity results are stated in terms of Chernoff information \citep{cover2012elements}.

\textbf{Chernoff Information}: Consider two Bernoulli distributions
$\mathcal{B}(x)$ and $\mathcal{B}(y)$, and let $d(x,y)$ denote the KL-divergence
between these distributions.  The Chernoff information $d^\ast(x,y)$ between
these two Bernoulli distributions is defined by 
\begin{equation*}
\vspace{-3pt}
\begin{small}
    \chernoff{x,y} := d(z^\ast,x) = d(z^\ast,y)
\end{small}
\vspace{-3pt}
\end{equation*}
where $z^\ast$ is the unique $z$ such that $d(z,x) = d(z,y)$.

Next we introduce some notation. For an arm $a$, let $g(a)$ (read group of arm
$a$) denote the index of the cluster that arm $a$ belongs to. Formally,
\begin{equation}
\vspace{-3pt}
\begin{small}
    g(a) := \min \{1 \le i \le c: p_a \le p_{\kappa_i}\}.
    \label{eq:clusterofarma}
\end{small}
\end{equation}
%We first introduce some notation. For an arm $a$, let $l(a)$ and $r(a)$ denote
%the indices of its its left and right cluster boundaries respectively, with the
%caveat that $l(a)$ is not defined for arms $1 \le a \le \kappa_1$ in the first
%cluster, and $r(a)$ is not defined for arms $\kappa_{c-1} < a \le K$ in the last
%cluster. Formally, 
%\begin{align}
%    \begin{aligned}
%    l(a) &= \max\{1 \le i \le (c-1): \kappa_i < a\} \text{ for }\kappa_1 < a
%    \le K \\
%    r(a) &= \min \{1 \le i \le (c-1): \kappa_i \ge a\} \text{ for }1 \le a \le
%    \kappa_{c-1}
%    \end{aligned}
%    \label{eq:leftrightclusters}
%\end{align}
%
Let $b_i \in [p_{\kappa_i}, p_{\kappa_i+1}], 1 \le i \le c-1$ be any points
in the cluster boundary gaps, and $b := (b_1,b_2,\dots,b_{c-1})$. Define 
\par\nobreak
\vspace{-18pt}
{\small 
\begin{align}
    \Delta^\ast_{b}(a) := 
    \begin{cases*}
        \chernoff{p_a,b_{1}} &\hspace{-65pt} $a\in \{1,\dots,\kappa_1\}$ \\
        \min(\chernoff{p_a,b_{g(a)-1}}, \chernoff{p_a,b_{g(a)}} &\\
            & \hspace{-62pt}$a \in \{\kappa_1+1,\dots,\kappa_{c-1}\}$ \\
            \chernoff{p_a,b_{c-1}} & \hspace{-65pt} $a \in \{\kappa_{c-1}+1,\dots,K \}$
    \label{eq:chernoffdistances}
    \end{cases*}
\end{align}
}%
to be the ``distance'' of each arm from the closest cluster boundary. %(see
%\cref{fig:upperandlowerbounddistances}). 
Our upper bound on the sample complexity of $\lucbrank$ is stated in \cref{thm:samplecomplexity}, and contains the
quantity $\hstar$ where
\vspace{-5pt}
\begin{equation}
\begin{small}
    \hstar := \sum\limits_{a \in \{1,\dots,K\}}^{} \frac{1}{\max
(\Delta^\ast_{b}(a), \epsilon^2/2 )}.
\label{eq:upperboundparameter}
\end{small}
\end{equation}
\begin{theorem}
    Let $b = (b_1,b_2,\dots,b_{c-1})$, where $b_i \in [p_{\kappa_i},
    p_{\kappa_i+1}]$. Let $\epsilon>0$. Let $\beta(t,\delta) = \log(\tfrac{k_1K
    t^\alpha}{\delta}) + \log \log (\tfrac{k_1K t^\alpha}{\delta})$ with $k_1 >
    \left( \frac{c-1}{2} \right)^\alpha + \tfrac{2e}{\alpha-1} +
    \tfrac{4e}{(\alpha-1)^2}$. Let $\tau$ be the random number of samples taken
    by $\lucbrank$ before termination. If $\alpha > 1$, 
        \[
    \begin{small}
            \P\left(\tau \le
        2C_0(\alpha) H^\ast_{\epsilon,b} \log \left( \frac{k_1 K
    (2 H^\ast_{\epsilon, b})^\alpha}{\delta} \right)\right) \ge 1-\delta
    \end{small}
\]
    %Moreover, for $\alpha > 2$, we have the following upper bound on
    %$\EE{\tau}$:
    %\[\EE{\tau} \le 2 C_0(\alpha) \sum_{i=1}^{c-1}H^\ast_{\epsilon, c_i} \log
    %\left( \frac{k_1 K (c-1) (2H^\ast_{\epsilon, c_i})^\alpha}{\delta} \right) +
    %K_\alpha \]
    where
    %\[K_\alpha = \frac{\delta}{k_1 K (\alpha-2)} + \frac{\delta
    %2^{\alpha-1}}{k_1 (\alpha-2)} + \frac{2^{\alpha-1} \delta}{k_1}\left( 1+
    %\frac{1}{e} + \log 2^\alpha + \log (1+\alpha \log 2) \right)
    %\frac{1}{(\alpha-2)^2},\] and
    $C_0(\alpha)$ is such that $C_0(\alpha) \ge \left( 1+\frac{1}{e} \right) \left( \alpha \log (C_0(\alpha)) + 1 +
    \frac{\alpha}{e} \right)$.
    \label{thm:samplecomplexity}
\end{theorem}
%\begin{figure}[t]
%    \centering
%    \includegraphics[width=\textwidth]{./Figures/upperlowerbound.pdf}\\[-.05in]
%    \hspace{.1in} {\tiny (a)} \hspace{3in} {\tiny (b)} \vspace{-0.05in}
%    \caption{Distances in the (a) upper, (b) lower bound for the sample
%    complexity ($\kappa_1=5, \kappa_2=15)$.}
%    \label{fig:upperandlowerbounddistances}
%\end{figure}

%!TEX root = Paper.tex

\subsection{Distribution-Dependent Lower Bound}
\label{sec:lower bound}
In this section, we state our non-asymptotic lower bound on the expected number of
samples needed by any $\delta$-PAC algorithm to cluster and rank the arms into
groups of sizes $(\kappa_1, \kappa_2-\kappa_1, \dots, K-\kappa_{c-1})$.
For simplicity, we focus on the case $\epsilon=0$. The proof of the lower bound 
uses standard change of measure arguments \citep{kaufmann2015complexity}, which requires some
continuity and well-separation assumptions. We state these next.

We consider the following class of bandit models where the clusters are
unambiguously separated, i.e.
\begin{equation}
\begin{small}
    \mathcal{M}_\kappa = \{p = (p_1,\dots,p_K): p_{i} \in \mathcal{P},
    p_{\kappa_i} > p_{\kappa_i+1}, 1 \le i < c\},
    \label{eq:separable}
\end{small}
\end{equation}
where $\mathcal{P}$ is a set that satisfies
\[
\begin{small}
    \forall\,p,q \in \mathcal{P}^2, p \ne q \Rightarrow 0 < KL(p,q) < +\infty.
\end{small}
\]
We also assume the following:
\begin{assumption} 
    For all $p,q \in \mathcal{P}^2$ such that $p \ne q$, for all $\alpha > 0$,

    there exists $q_1\in\mathcal{P}$: $\KL{p,q} < \KL{p,q_1}<\KL{p,q}+\alpha$
    and $\mathbb{E}_{X\sim q_1}[X]>\mathbb{E}_{X\sim q}[X]$,

    there exists $q_2\in\mathcal{P}$: $\KL{p,q} < \KL{p,q_2}<\KL{p,q}+\alpha$
    and $\mathbb{E}_{X\sim q_2}[X]<\mathbb{E}_{X\sim q}[X]$.
    \label{ass:assumption}
\end{assumption}
To state our lower bound, we need to define for each arm $a$, another
``distance'' from the boundary, similar to \eqref{eq:chernoffdistances}.
Define

\vspace{-17pt}
{\small
\begin{align}
    \Delta^\text{KL}_{\kappa}(a) := 
    \begin{cases*}
        \KL{p_a,p_{\kappa_1+1}} &\hspace{-95pt} $a\in \{1,\dots,\kappa_1\}$ \\
        \min(\KL{p_a,p_{\kappa_{g(a)-1}}}, \KL{p_a,p_{\kappa_{g(a)}+1}} &\\
            &\hspace{-95pt} $a \in \{\kappa_1+1,\dots,\kappa_{c-1}\}$ \\
            \KL{p_a,p_{\kappa_{c-1}}} &\hspace{-95pt} $a \in \{\kappa_{c-1}+1,\dots,K \},$
    \end{cases*}
    \label{eq:kldistances}
\end{align}
}%
where $g(a)$ defined in \eqref{eq:clusterofarma} is the cluster that arm
$a$ belongs to.
We highlight the differences from \eqref{eq:chernoffdistances}. First, the Chernoff
information in \eqref{eq:chernoffdistances} is replaced with KL-divergence in
\eqref{eq:kldistances}, and second,
the distance is measured with the closest arm in either adjacent cluster here, as
opposed to a point in the gap between the clusters in
\eqref{eq:chernoffdistances}. %(see
%\cref{fig:upperandlowerbounddistances}).

Our lower bound involves the quantity 
\begin{equation}
\begin{small}
    \sum\limits_{a \in {1,\dots,K}}^{} \frac{1}{\Delta^{\text{KL}}_\kappa(a)}
    \label{eq:lowerboundparameter}
\end{small}
\end{equation}
and is as follows: 
\begin{theorem}
    Let $p \in \mathcal{M}_\kappa$, and assume that $\mathcal{P}$ satisfies
    \cref{ass:assumption}; any coarse ranking algorithm that is $\delta$-PAC on
    $\mathcal{M}_\kappa$ satisfies, for $\delta \le 0.15$,
    \vspace{-5pt}
    \[
    \begin{small}
        \mathbb{E}_p[\tau] \ge \left[ \sum\limits_{a \in {1,\dots,K}}^{}
        \frac{1}{\Delta^{\text{KL}}_\kappa(a)} \right] \log \left(
            \frac{1}{2.4\delta} \right)
    \end{small}
    \]
    \label{thm:lowerbound}
\end{theorem}
\vspace{-.3in}
\subsection{Remarks}
\vspace{-.1in}
\label{sec:discussion}
\begin{itemize} 
    \item The tightest high-probability upper bound is obtained by setting $b$ equal to
        $\argmin\limits_{b:b_i\in [p_{\kappa_i}, p_{\kappa_i+1}]} \hstar$ in
        \cref{thm:samplecomplexity}.
\item Although stated for Bernoulli distributions, the results in this paper can
    easily be extended to rewards in the exponential family
    \citep{garivier2011kl} by using the appropriate $d$ function.
\end{itemize}
\vspace{-.1in}

%!TEX root = Paper.tex

\section{Experiments}
\label{sec:experiments}
\subsection{Ranking from Direct Feedback}
\label{sec:experiment1}
\begin{figure}[ht]
    \centering
    \includegraphics[width=.5\columnwidth]{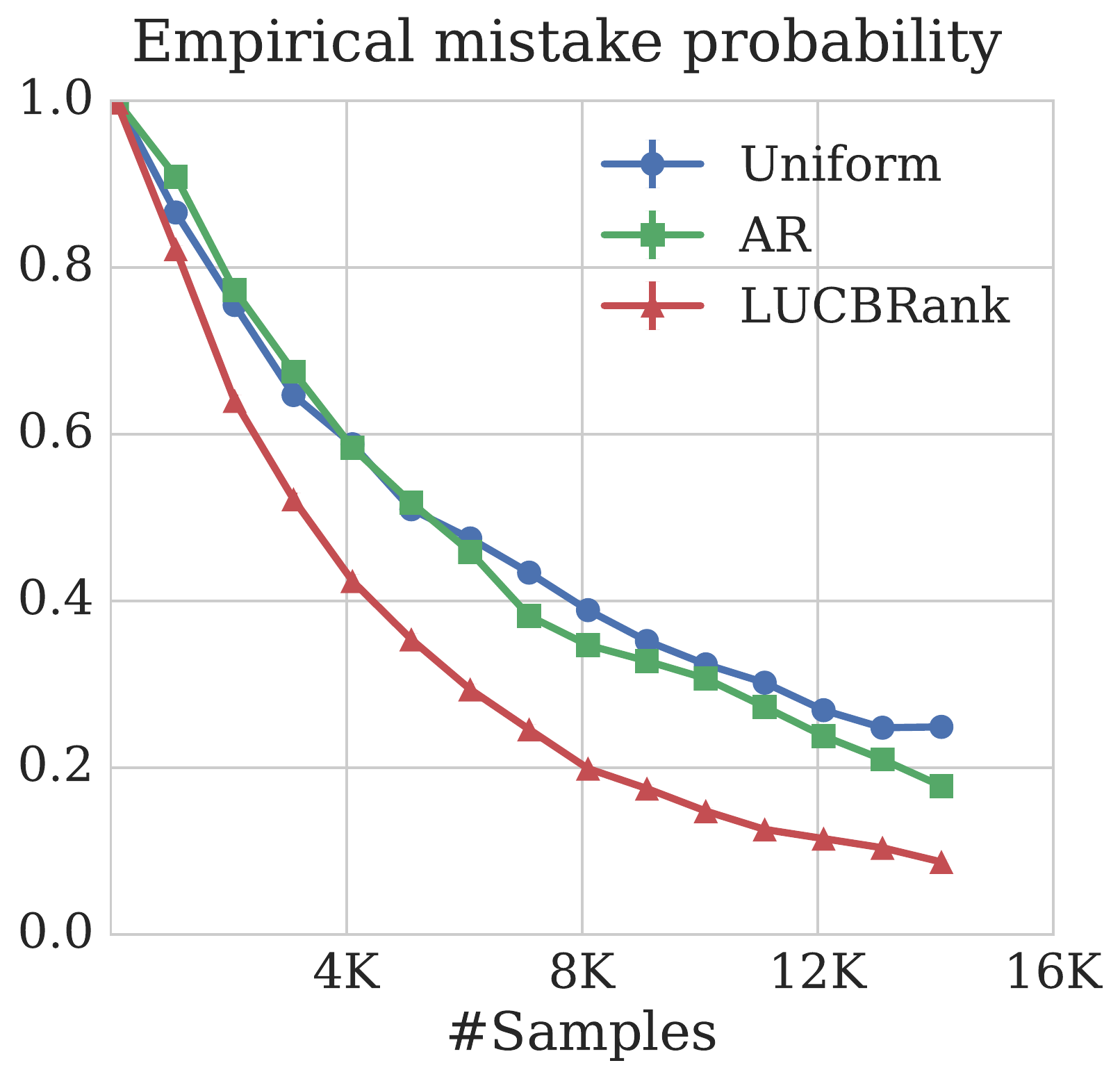}
    \caption{Exp 1 (description in text)}
    \label{fig:expt1}
\end{figure}
We first compare $\lucbrank$ with uniform sampling and the Active Ranking (AR) algorithm \citep{heckel2016active}. AR is an adaptation of the successive elimination approach to solve the coarse ranking problem. It maintains a set of unranked items and samples every item in this set, removing an item from the set when it is confident of the cluster the item belongs to. Although developed for pairwise comparison feedback, AR can easily be adapted to the direct-feedback setting.

We look at the bandit instance $B$ with $K=15$ arms whose rewards are Bernoulli distributed with means $(p_1=\tfrac{1}{2}; p_a=\tfrac{1}{2}-\tfrac{a}{40}$ for $a=2,3,\dots,K)$. This problem has been studied in the literature in the context of finding the best-arm \citep{bubeck2013multiple}. We consider the problem of finding the top-$3$ and the bottom-$3$ arms, which corresponds to $\kappa_1=3, \kappa_2=12$.

In \cref{fig:expt1}, we record the probability (averaged over $1000$ simulations) that the empirical clusters returned by the algorithm do not match the true clusters. We set $\delta=0.1$ for both $\lucbrank$ and AR, and $\epsilon=0$ in $\lucbrank$ to have a fair comparison with AR. We see that the mistake probability drops faster for $\lucbrank$ than for AR.
\subsection{Ranking from Pairwise Comparisons}
\label{sec:chicagostreetview}
\vspace{-10pt}
\begin{figure}[ht]
    \centering
    \subfloat[(a)]{
        \includegraphics[width=5cm]{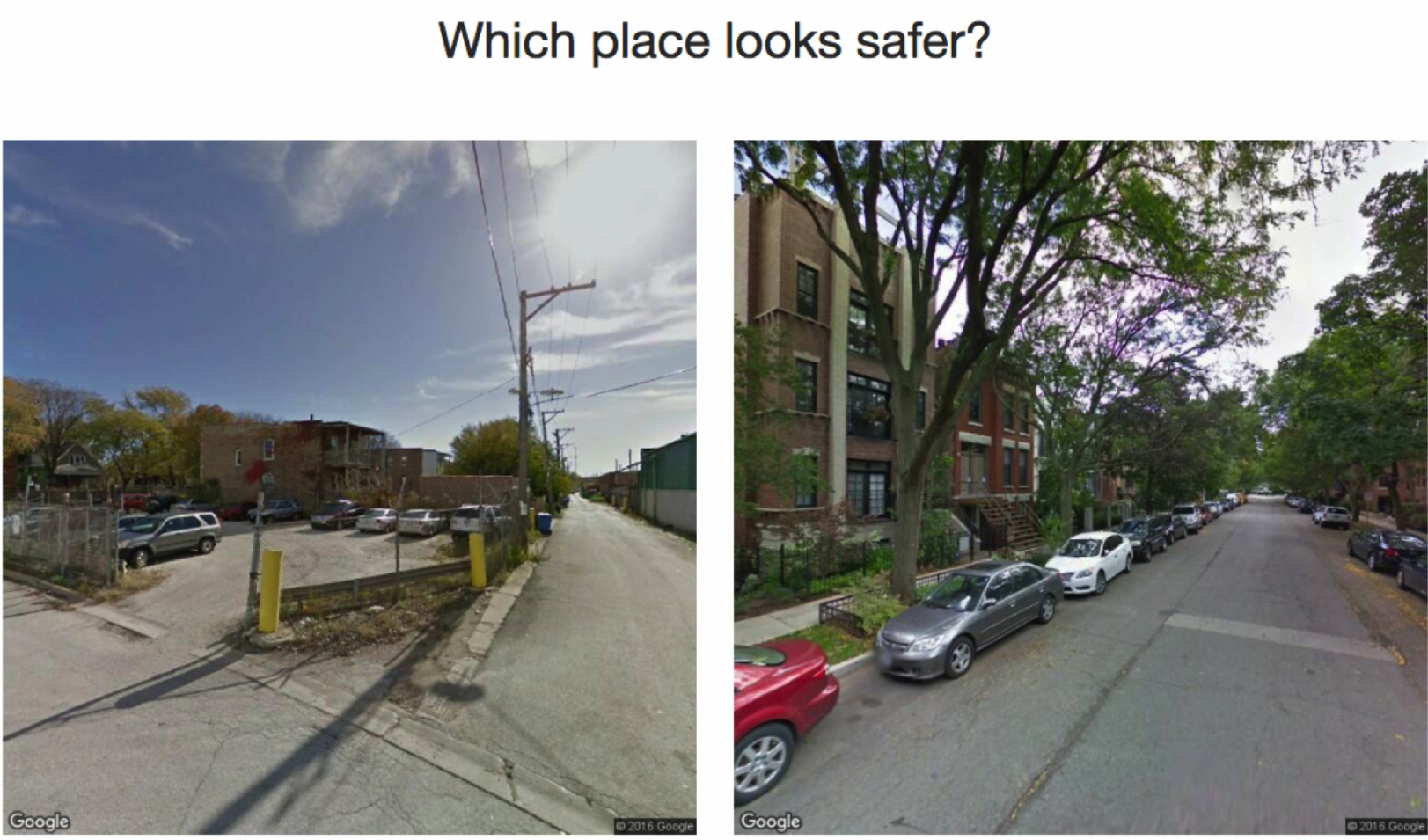}
    }\par
%    \subfloat[0.05\hspace{35pt} 1.09\hspace{10pt}(b)\hspace{10pt} 2.98\hspace{35pt}
%    3.77]{
    \subfloat[\hspace{10pt}0.05\hspace{25pt} 1.09\hspace{25pt} 2.98\hspace{25pt} 3.77
    \newline (b)]{
        \includegraphics[width=1.5cm]{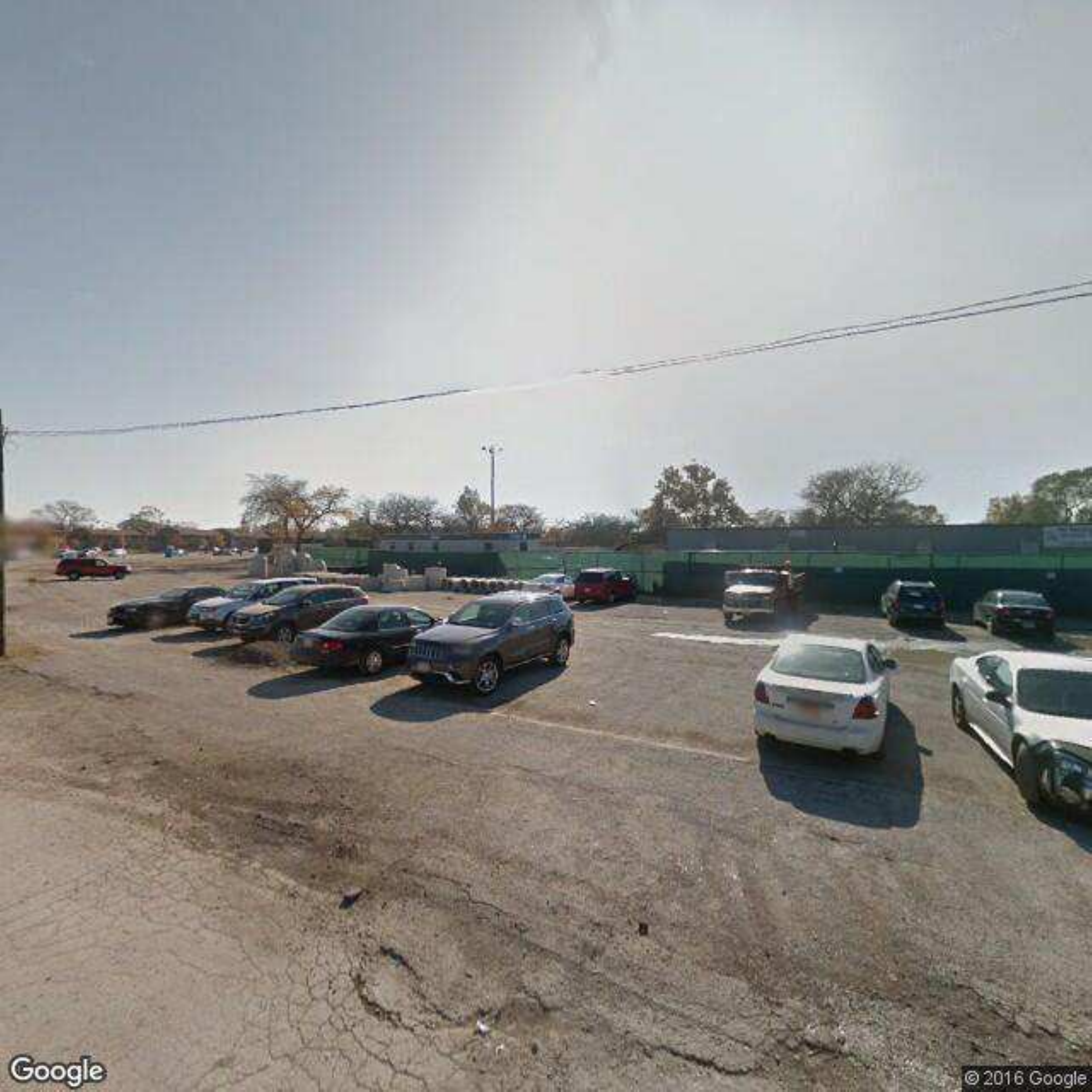}
        \includegraphics[width=1.5cm]{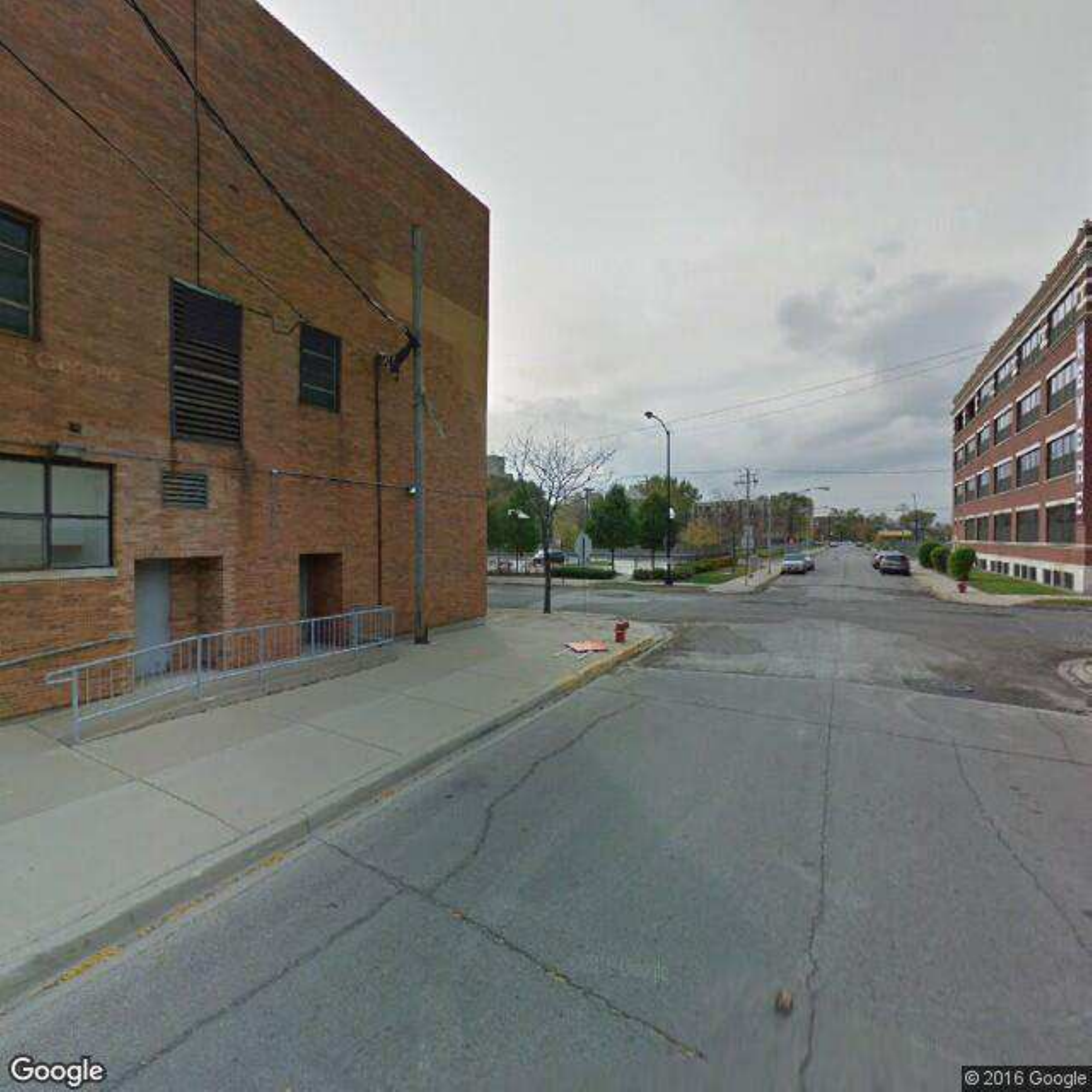}
        \includegraphics[width=1.5cm]{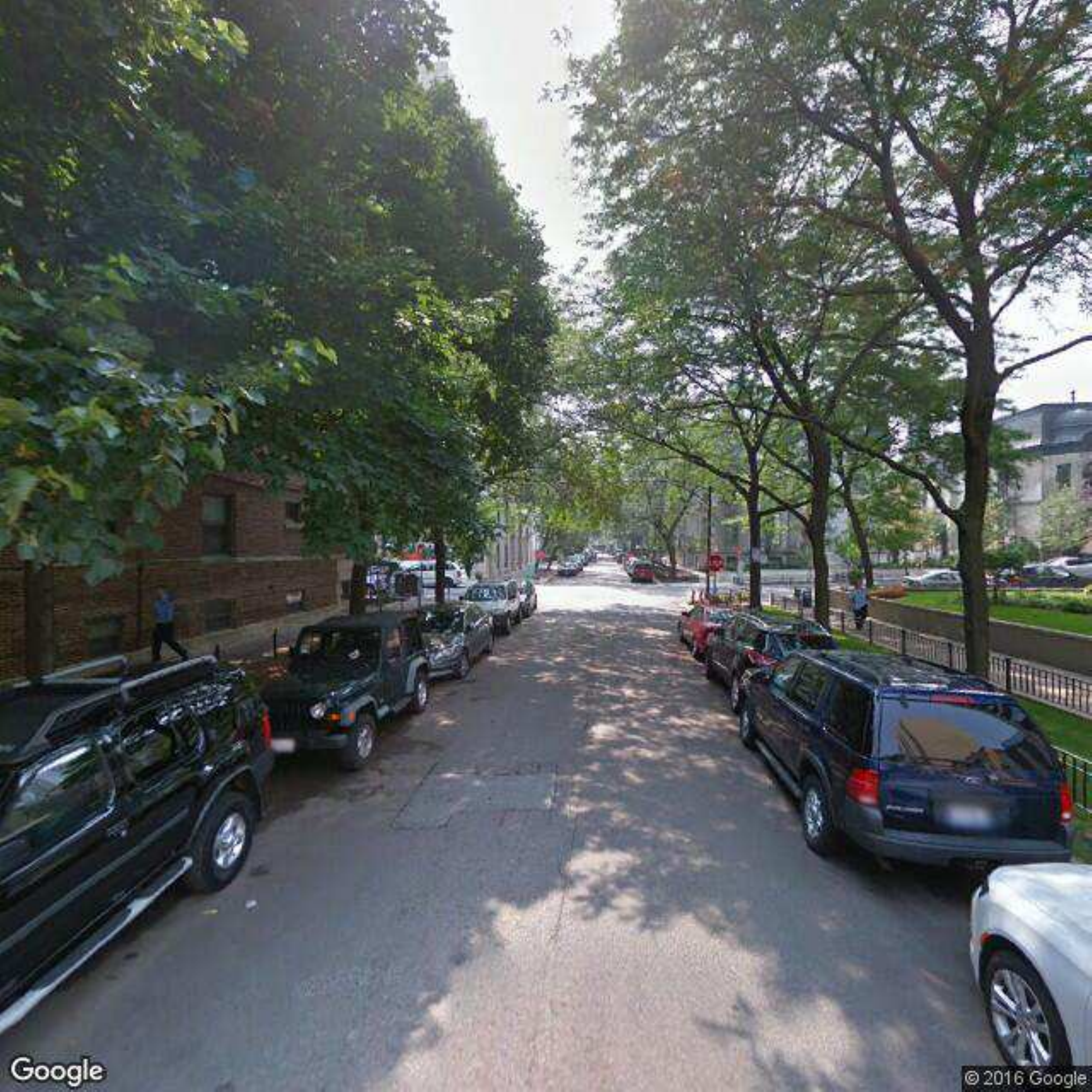}
        \includegraphics[width=1.5cm]{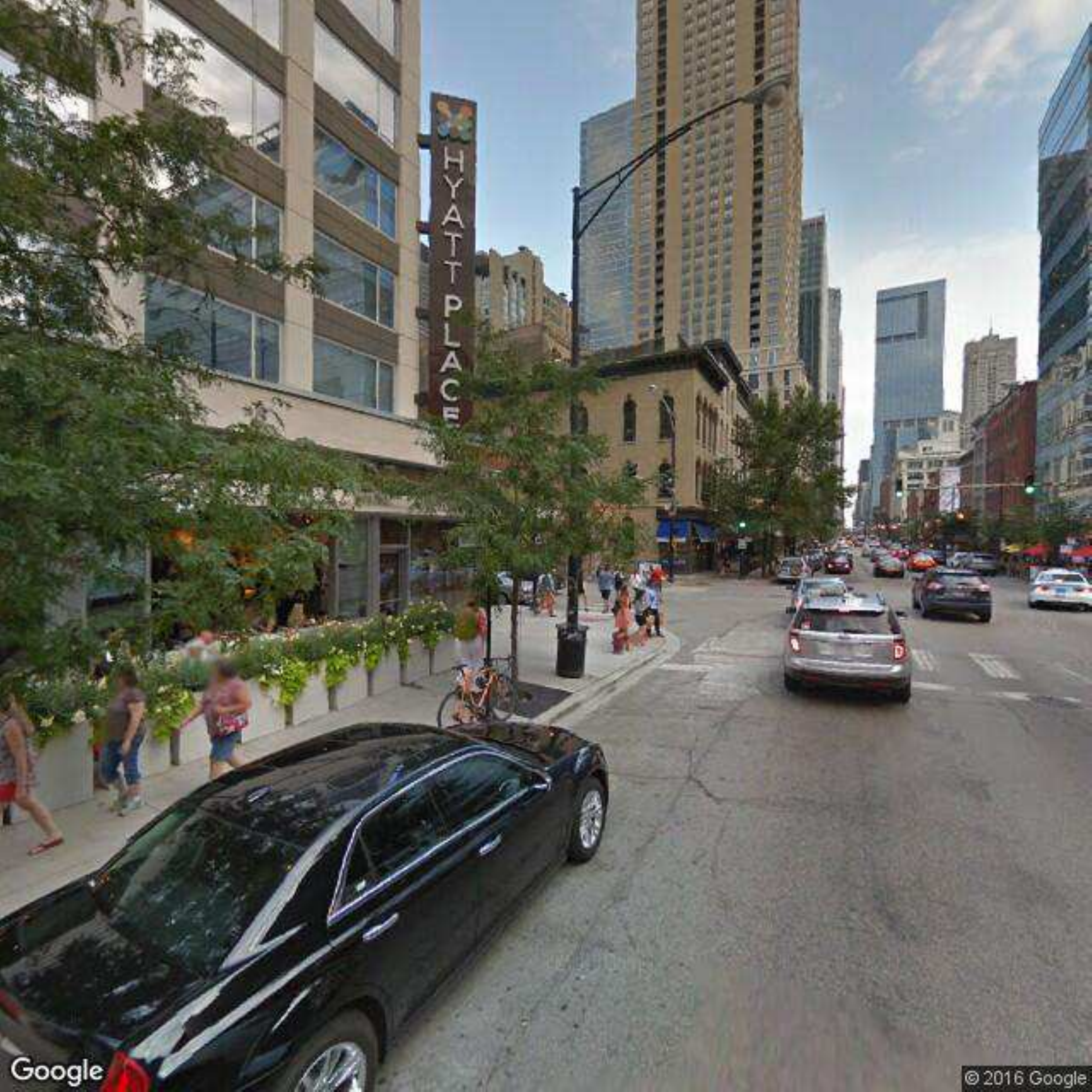}
    }\par
    \vspace{-20pt}
    \subfloat[(c)]{
        \includegraphics[width=6cm]{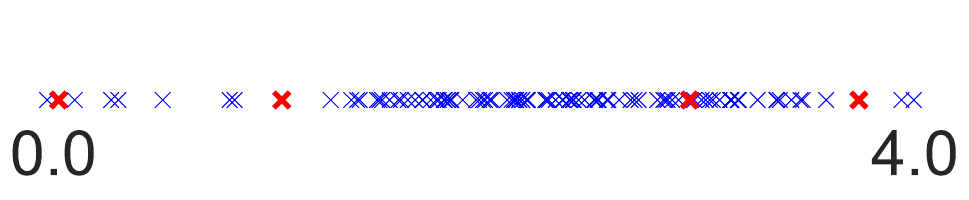}
    \vspace{-25pt}
    }
    \caption{(a) A sample query on NEXT. (b) Four sample images and their
    estimated BTL scores beneath. (c) Scatter plot of all the BTL scores, with the 
    sample image markers highlighted.}
    \label{fig:chicagoimages}
    \vspace{-10pt}
\end{figure}
To measure the performance of our algorithm on real-world data, we selected
$K=100$ Google street view images in Chicago, and collected $6000$ pairwise
responses on MTurk using NEXT \citep{jamieson2015next}, where we asked users to
choose the safer-looking image out of two images. This experiment is similar to
the Place Pulse project \citep{naik2014streetscore}, where the objective is to assess
how the appearance of a neighborhood affects its perception of safety.
\cref{fig:chicagoimages}(a) shows a sample query from our experiment.
We estimated the safety scores of these street view images from the
user-responses by fitting a Bradley-Terry-Luce (BTL) model
\citep{bradley1952rank} using maximum likelihood estimation, and used this as the ground truth to generate noisy comparisons. Given two items $i$ and $j$ with scores $\theta_i$ and
$\theta_j$, the BTL model estimates the probability that item $i$ is preferred
to item $j$ as $\P(i > j) = \frac{e^{\theta_i}}{e^{\theta_i} + e^{\theta_j}}$.
\cref{fig:chicagoimages}(b) shows $4$ images overlayed with their estimated BTL
scores (where the lowest score was set to $0$), and \cref{fig:chicagoimages}(c)
shows a scatter plot of the scores of all $100$ images. 

\begin{figure}[ht]
    \centering
    \includegraphics[width=.31\textwidth]{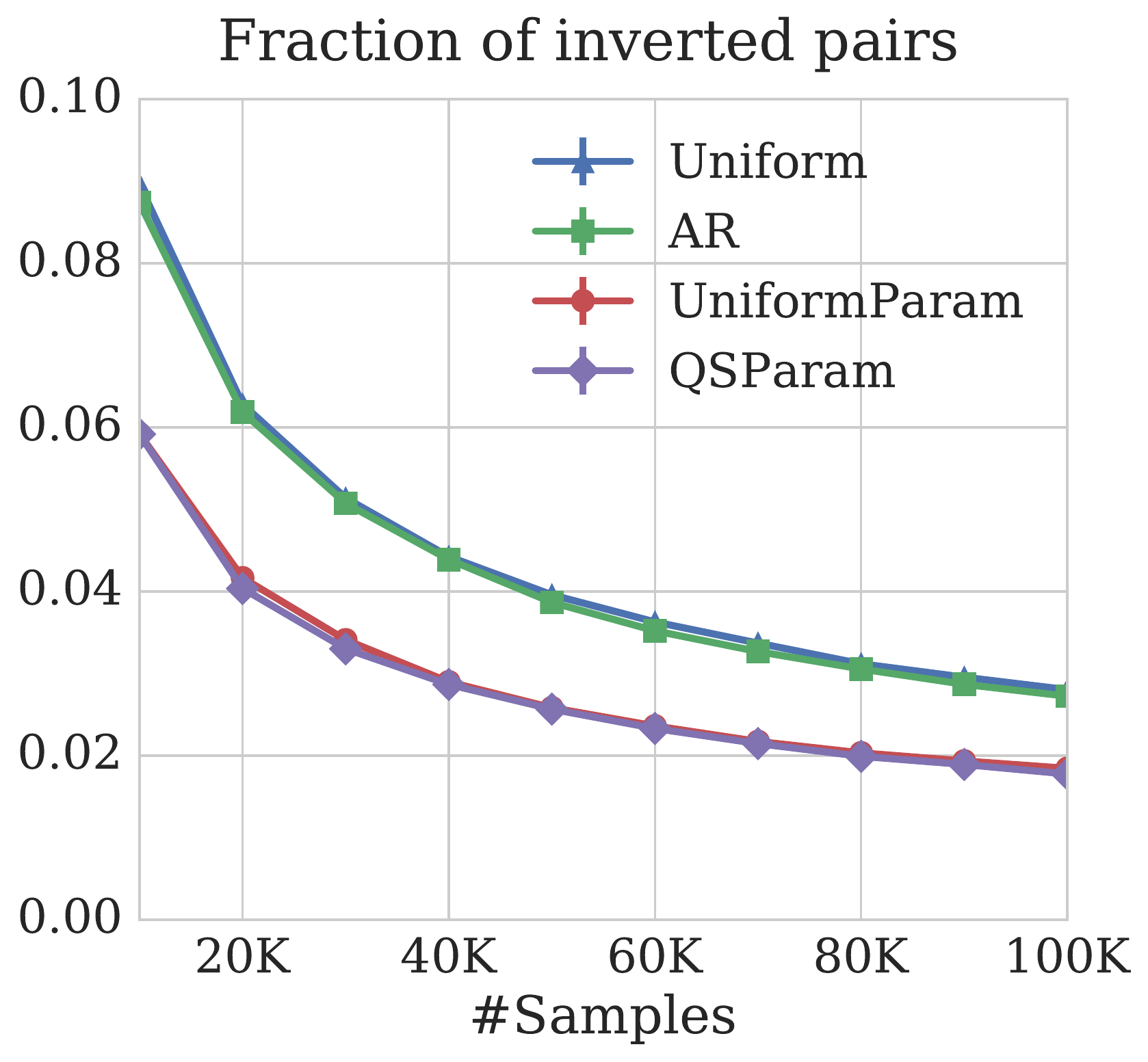}
    %\hspace{.2in}
    %\includegraphics[width=.3\textwidth]{./Figures/chicago_mistake.pdf}\\[-.1in]
    %\hspace{.1in} {\tiny (a)} \hspace{1.8in} {\tiny (b)} \vspace{-0.05in}
    \caption{The futility of adaptive methods if the goal is to
    obtain a complete ranking. We compare uniform sampling with Active Ranking (both use
non-parametric rank aggregation), and uniform sampling with quicksort (where
both use parametric rank aggregation). We see that the Kendall tau distance of
adaptive methods is no lower than those of their non-adaptive counterparts.}
    \label{fig:completerankingisfutile}
    \vspace{-2pt}
\end{figure}
We first study the performance of adaptive methods with the goal of finding a \emph{complete ranking}, and observe that adaptive methods offer no advantages when
items means are close to each other as they are in this dataset. Oblivious of
the generative model, a lower bound (ignoring constants and log factors) on the
number of samples required to sort the items by their Borda
scores is given by $\sum 1 / \Delta_i^2$ \citep{jamieson2015sparse}, where the $\Delta_i$s are gaps between
consecutive sorted Borda scores. For the dataset
considered in this experiment, $\sum 1/\Delta_i^2 = 322$ million! We verify the
futility of adaptive methods in \cref{fig:completerankingisfutile}, where we compare
the performance of parametric as well as non-parametric adaptive methods in the
literature (we describe these methods shortly) to their non-adaptive
counterparts, with a goal of finding a complete ranking of the images. In the
parametric algorithms (UniformParam and QSParam), we find MLE estimates of the
BTL scores that best fit the pairwise responses. In the non-parametric
algorithms (Uniform and AR), we estimate the scores using empirical
probabilities in \cref{eq:bordadefn}. In \cref{fig:completerankingisfutile}, we plot
the fraction of pairs that are inverted in the empirical ranking compared to the
true ranking, and see no benefits for adaptive methods. We do see gains from
adaptivity in the coarse formulation (\cref{fig:chicagoresults}),
as we explain next.

$\lucbrank$ can be used in the pairwise comparison setting using Borda reduction, as described in \cref{sec:ranking pc}. The adaptive methods in literature we compare to are AR (as in the previous section), and Quicksort (QS) \citep{ailon2008aggregating, maystre2017just}. The Quicksort algorithm works exactly like its non-noisy counterpart: it compares a randomly chosen pivot to all elements, and divides the elements into two subsets - elements preferred to the pivot, and elements the pivot was preferred over.  The algorithm then recurses into these two subsets. In this experiment, we stop the quicksort algorithm early as soon as all the subsets are inside the user-specified clusters. Continuing the algorithm further won't change the items in any cluster. This reduces the sample complexity of Quicksort.

\begin{figure}[ht]
    \centering
    %\includegraphics[width=.31\textwidth]{./Figures/complete_ranking_inverted_pairs.pdf}
    %\hspace{.2in}
    \includegraphics[width=.25\textwidth]{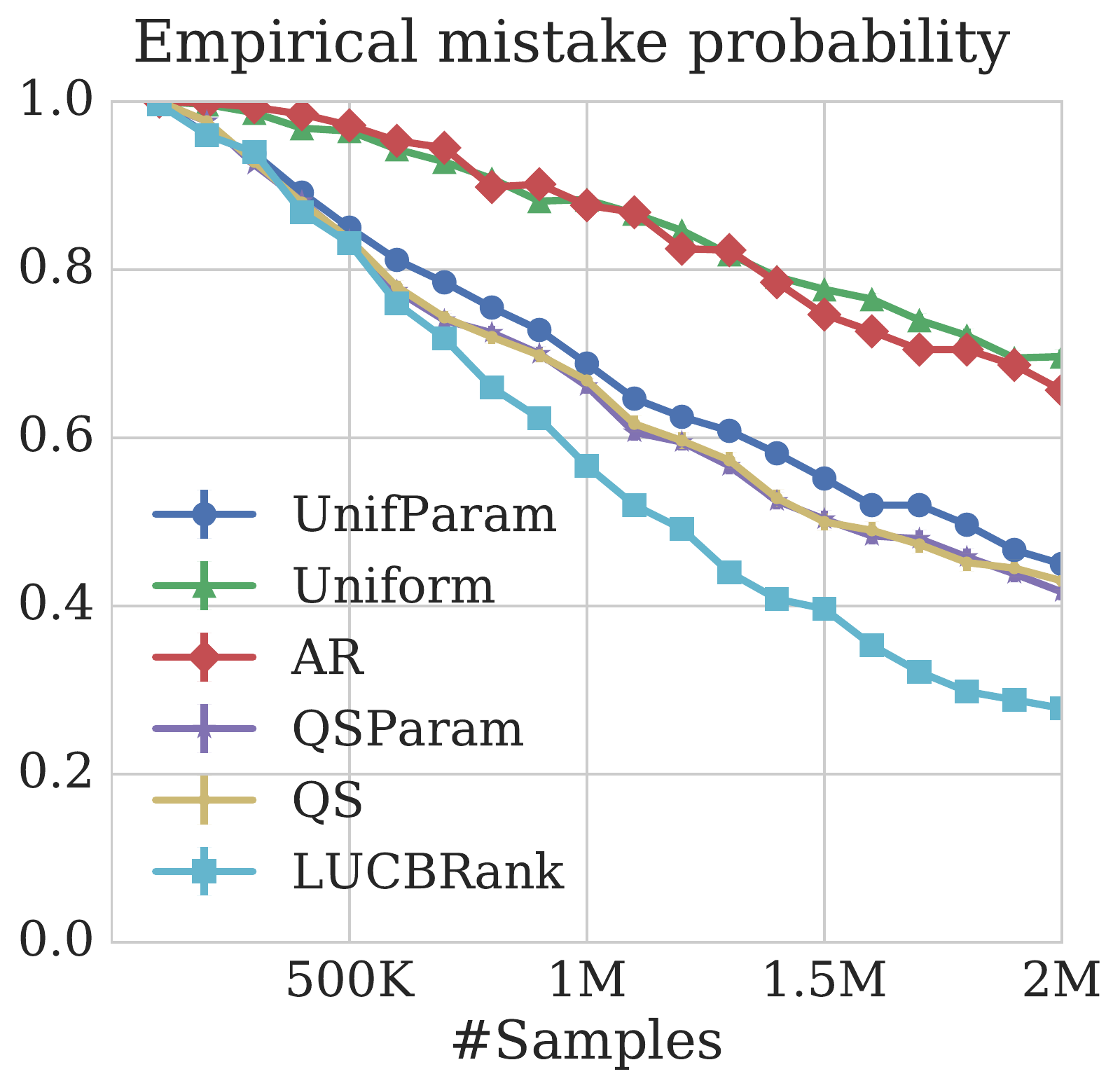}
    %\hspace{.1in} {\tiny (a)} \hspace{1.8in} {\tiny (b)} \vspace{-0.05in}
    \caption{Probability of error in identifying the clusters: $\lucbrank$
    does better than parametric versions of other active algorithms.}
    \label{fig:chicagoresults}
\vspace{-5pt}
\end{figure}
We consider the problem of clustering the images into pentiles ($\kappa_i =
20\,i,\;1\le i\le5$). We set $\delta=0.1$ for both $\lucbrank$ and AR, and
$\epsilon=0$ in $\lucbrank$ to ensure a fair comparison with AR. In
\cref{fig:chicagoresults}, we record the probability (averaged over $600$
simulations) that the empirical pentiles returned by the algorithm do not match
the true pentiles. We find that $\lucbrank$ has a lower mistake probability than
even the parametric version of Quicksort, which assumes knowledge of the BTL model. As an aside, note that when the items
are close as in this experiment, the parametric versions of Uniform and
Quicksort perform similarly, and the active nature of Quicksort offers no
significant advantage. 

%\begin{figure}[t]
%    \centering
%    \subfloat[(a)]{
%        \includegraphics[width=.32\textwidth]{./Figures/inter-cluster-inversions.pdf}
%    }\par
%    \subfloat[(b)]{
%        \includegraphics[width=.32\textwidth]{./Figures/intra-cluster-inversions.pdf}
%    }
%    %\hspace{.2in} {\tiny (a)} \hspace{1.7in} {\tiny (b)} \vspace{-0.05in}
%    \caption{$\lucbrank$ has (a) lower inter-cluster and (b) higher intra-cluster
%        pairs inverted as compared to Uniform: it trades intra-cluster
%    inversions in favor of minimizing inter-cluster inversions.}
%    %\vspace{-.1in}
%    \label{fig:chicagoresults2}
%\end{figure}
%\vspace{-10pt}
\begin{figure}[ht]
    \centering
    \includegraphics[width=.6\columnwidth]{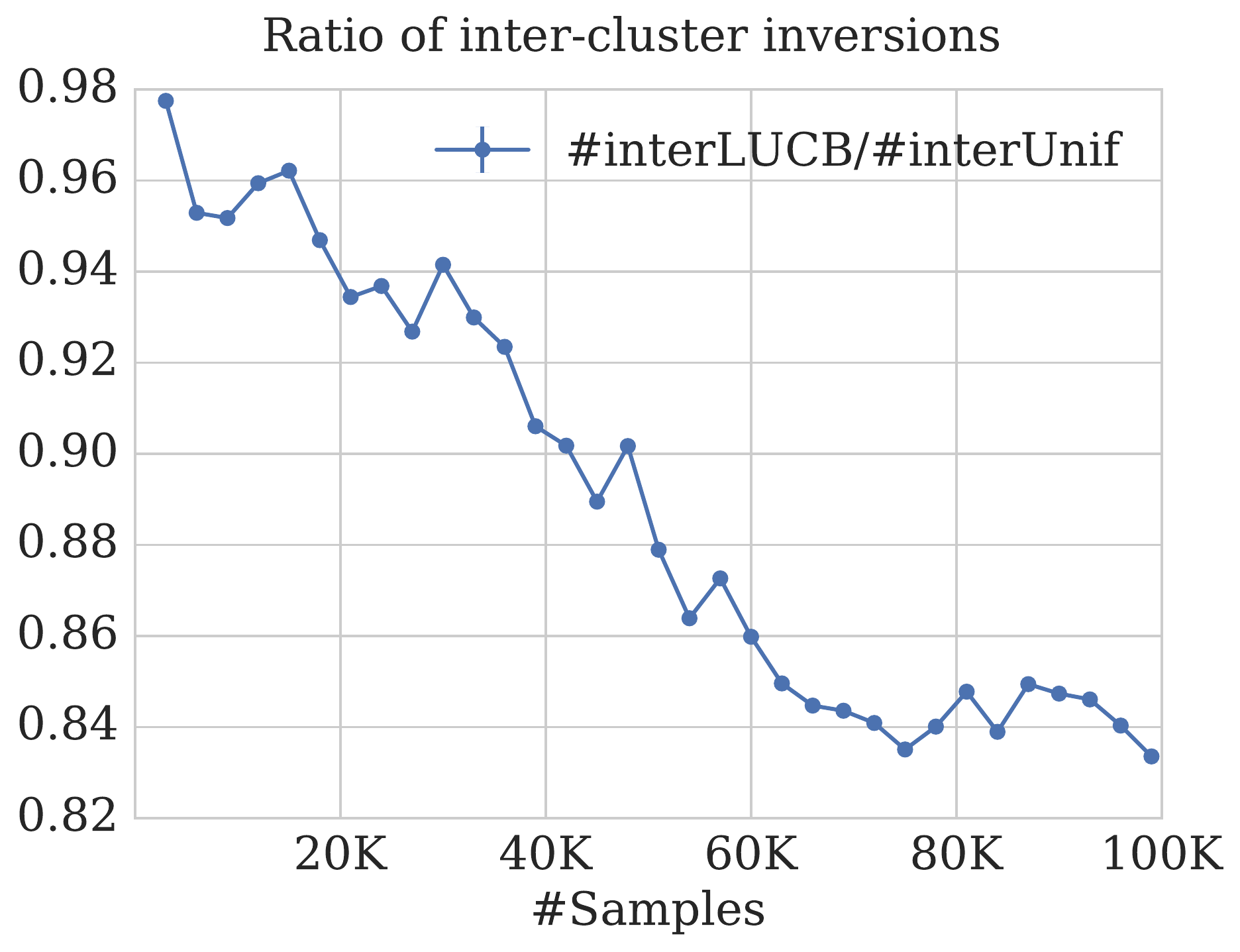}
    %\hspace{.2in}
    \includegraphics[width=.6\columnwidth]{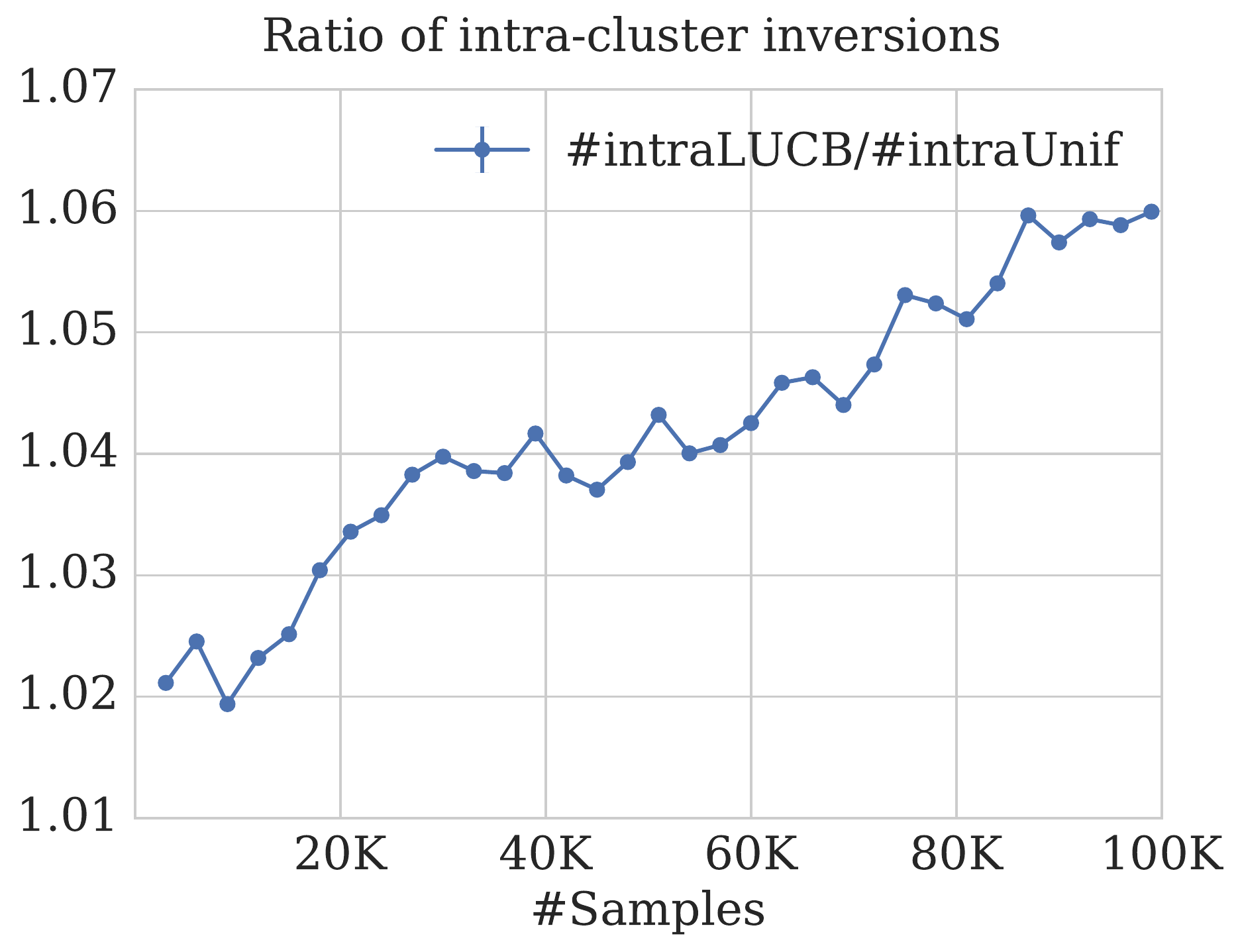}\\[-.1in]
    %\hspace{.2in} {\tiny (a)} \hspace{1.7in} {\tiny (b)} \vspace{-0.05in}
    \caption{(a) The ratio of inter-cluster inversions of $\lucbrank$ and
        Uniform. (b) The ratio of intra-cluster inversions of $\lucbrank$ and
        Uniform. $\lucbrank$ focuses on minimizing inter-cluster inversions at
    the cost of intra-cluster inversions.}
    \vspace{-7pt}
    \label{fig:chicagoresults2}
\end{figure}
In \cref{fig:chicagoresults2}(a) and (b) we plot the ratio of inter-cluster and
intra-cluster inversions respectively of $\lucbrank$ and Uniform.  An
inter-cluster pair is a pair of items that are in different clusters in the true
ranking, while an intra-cluster pair is a pair of items from the same cluster.
We see the that ratio of inter-cluster inversions goes down in
\cref{fig:chicagoresults2}(a), because that is the metric $\lucbrank$ focuses
on. $\lucbrank$ does not expend effort on refining its estimate of an item's rank once its cluster has
been found, and hence pays a price in the form of intra-cluster inversions
(\cref{fig:chicagoresults2}(b)).
%This clearly shows that $\lucbrank$ focuses its effort on getting the global
%position right, at the expense of small local errors.
%\vspace{-10pt}

\section{Conclusion}
\label{sec:conclusion}
\vspace{-10pt}
The coarse ranking setting is motivated from real-world problems where humans rate items. These problems have high noise and are hard, and a complete ranking is not feasible; fortunately, it is often also not necessary. We propose a practical online algorithm for solving it, $\lucbrank$, and prove distribution-dependent upper and lower bounds on its sample complexity. We evaluate its performance on crowdsourced data gathered using MTurk, and observe that it performs better than existing algorithms in the literature.

We leave open several questions. First, our upper bound is stated in terms of Chernoff information between distributions, while our lower bound is in terms of KL-divergences, and there is a gap between the two. Second, the cluster boundaries need to be user-specified in our current setting. If the gap between the nearest items in adjacent clusters is small, this can adversely affect the sample complexity. Although this is partially addressed through the error-tolerance $\epsilon$, an attractive algorithm would be one which auto-tunes the positions of the cluster boundaries at the widest gaps, subject to user-specified constraints.

To the best of our knowledge, this paper presents the first bandit UCB algorithm for ranking.

\vspace{-10pt}
\section*{Acknowledgements}
\vspace{-10pt}
The authors would like to thank Scott Sievert and Xiaomin Zhang for help with experiments, and Ervin Tanczos for discussions.
\bibliography{References}
\clearpage
%!TEX root = Paper.tex

\onecolumn
\section{Appendix}
\label{sec:appendix}

%\subsection{Proof of \cref{thm:pacguarantee}}
\subsection{PAC Guarantee}
We'll use the following lemma \citep{kaufmann2013information} which bounds
the probability of `bad' events in round $t$.
\begin{lemma}
    Let $U_a(t)$ and $L_a(t)$ be the confidence bounds defined in
    \cref{eq:klupperbound}. For any algorithm and arm
    $a$, 
    \[\P(U_a(t) < p_a) = \P(L_a(t) > p_a) \le e (\beta(t,\delta) \log t + 1)
    \exp(-\beta(t,\delta))\]
    \label{lem:concentrationinequality1}
\end{lemma}
We shall also need the following technical lemma, which we'll use to upper bound
the probability of any bad event.
\begin{lemma}
    If $\beta(t,\delta) = \log(\tfrac{k_1 K t^\alpha}{\delta}) + \log \log
    (\tfrac{k_1 K t^\alpha}{\delta})$, 
    \[\sum\limits_{t=1}^{\infty} (\beta(t,\delta) \log t + 1)
    \exp(-\beta(t,\delta)) \le \frac{\delta}{k_1 K} \left( \frac{2}{(\alpha-1)^2} +
\frac{1}{(\alpha-1)} \right)\]
    \label{lem:sumoferrorterms}
\end{lemma}
\begin{proof}
    Let us consider 
    \begin{align*}
        \beta(t,\delta) (\log t) e^{-\beta(t,\delta)} &= \left( \log \left(\tfrac{k_1
        K t^\alpha}{\delta} \right) + \log \log \left(\tfrac{k_1 K
        t^\alpha}{\delta} \right)
    \right) (\log t) \left( \frac{\delta}{k_1 K t^\alpha} \cdot \frac{1}{\log
    \frac{k_1 K t^\alpha}{\delta}} \right) \\
    &\le 2 \log \left( \frac{k_1 K t^\alpha}{\delta} \right) \cdot \log t \cdot
    \left( \frac{\delta}{k_1 K t^\alpha} \cdot \frac{1}{\log \frac{k_1 K
    t^\alpha}{\delta}} \right) \\
    &= 2 \log t \cdot \frac{\delta}{k_1 K t^\alpha}
    \end{align*}
    Hence
    \begin{align*}
    \sum\limits_{t=1}^{\infty} (\beta(t,\delta) \log t + 1)
    \exp(-\beta(t,\delta)) &\le \sum\limits_{t=1}^{\infty} \left( 2 \log t
    \cdot \frac{\delta}{k_1 K t^\alpha} + \frac{\delta}{k_1 K t^\alpha} \right)
    \\
     &\le \frac{\delta}{k_1 K} \left( \frac{2}{(\alpha-1)^2} +
    \frac{1}{(\alpha-1)} \right)
    \end{align*}
\end{proof}

\subsubsection{Proof of \cref{thm:pacguarantee}}
\begin{theorem*}
    $\lucbrank$ using $\beta(t,\delta) = \log(\tfrac{k_1 K t^\alpha}{\delta}) +
    \log \log (\tfrac{k_1 K t^\alpha}{\delta})$ with $\alpha > 1$ and $k_1 > 1 +
    \tfrac{2e}{\alpha-1} + \tfrac{4e}{(\alpha-1)^2}$, is correct with
    probability $1-\delta$.
\end{theorem*}
\begin{proof}
    Consider the event 
    \[W = \bigcap\limits_{t \in \mathbb{N}}^{} \bigcap\limits_{a \in
    \{1,\dots,K\}}^{} \left( (U_a(t) > p_a) \cap (L_a(t) < p_a) \right)\]
    where all arms are well-behaved i.e. their true means are
    inside their confidence intervals. We show that $\lucbrank$ is correct on
    the event $W$.

    Assume $\lucbrank$ fails, which means that when it terminates, there exists
    a cluster $i$, such that arm $a$ belongs to cluster $i$ in the returned ranking,
    and $a \in M^{\ast,c}_{\epsilon,i}$; that is, either 1) $p_a >
    p_{\kappa_{i-1}+1}+\epsilon$ or 2) $p_a < p_{\kappa_i}-\epsilon$. 
    
    Consider the first case: $p_a > p_{\kappa_{i-1}+1}+\epsilon$. Consequently,
    there exists arm $b$ such that $p_b \le p_{\kappa_{i-1}+1}$, and $\tau(b)
    \le \kappa_{i-1}$ in the returned ranking. Since the algorithm stopped and
    boundary $i-1$ was removed from the set of active boundaries $C$, it must be the
    case that $U_a(t)-L_b(t) < \epsilon$ upon stopping. Hence, the following holds:
    \begin{align*}
        &\bigcup\limits_{t \in \mathbb{N}}^{} (\exists\, a,b: p_a >
        p_{\kappa_{i-1}+1}+\epsilon, p_b \le p_{\kappa_{i-1}+1}, 
        U_a(t)-L_b(t) < \epsilon) \\
        \subseteq & \bigcup\limits_{t \in \mathbb{N}}^{} (\exists\,a,b:
        (U_a(t) < p_b+\epsilon < p_a) \cup (L_b(t) > p_b)) \\
        \subseteq &\bigcup\limits_{t \in \mathbb{N}}^{}
        \bigcup\limits_{a \in \{1,\dots,K\}}^{} (U_a(t) < p_a)
        \bigcup\limits_{b \in \{1,\dots,K\}}^{}(L_b(t) > p_b)
            \subseteq W^c 
    \end{align*}

    Consider the second case: $p_a < p_{\kappa_i}-\epsilon$. Consequently, there
    exists an arm $b$ such that $p_b \ge p_{\kappa_i}$, and $\tau(b) > \kappa_i$
    in the returned ranking. Since the algorithm stopped and boundary $i$ was
    removed from the set of active boundaries $C$, it must be the case that
    $U_b(t)-L_a(t) < \epsilon$ upon stopping. Hence, the following holds:
    \begin{align*}
        &\bigcup\limits_{t \in \mathbb{N}}^{} (\exists\, a,b: p_a <
        p_{\kappa_i}-\epsilon, p_b \ge p_{\kappa_i}, 
        U_b(t)-L_a(t) < \epsilon) \\
        \subseteq & \bigcup\limits_{t \in \mathbb{N}}^{} (\exists\,a,b:
        (U_b(t) < p_b) \cup (L_a(t) > p_b - \epsilon > p_a)) \\
        \subseteq &\bigcup\limits_{t \in \mathbb{N}}^{}
        \bigcup\limits_{b \in \{1,\dots,K\}}^{} (U_b(t) < p_b)
        \bigcup\limits_{a \in \{1,\dots,K\}}^{}(L_a(t) > p_a)
            \subseteq W^c 
    \end{align*}
 
    Hence
    %From \cref{lem:concentrationinequality1}, 
    \begin{alignat*}{2}
        \P(\lucbrank \text{  fails}) &\le \P(W^c) \quad & \nonumber\\
        &\le 2eK\sum\limits_{t=1}^{\infty} (\beta(t,\delta) \log t + 1)
        \exp(-\beta(t,\delta)) \quad &(\text{by
        \cref{lem:concentrationinequality1}}) \\
        &\le \frac{\delta}{k_1} \left( \frac{4e}{(\alpha-1)^2} +
        \frac{2e}{(\alpha-1)} \right)  &(\text{by \cref{lem:sumoferrorterms}})
        \\
        &\le \delta &(\text{by the constraint on }k_1)
        %\le \delta
    %\label{eq:boundonWc}
    \end{alignat*}
    %\end{equation}
\end{proof}
%
%
%\subsection{Proof of \cref{thm:samplecomplexity}}
\subsection{Sample Complexity}
We define the event $W_t$ which says that all arms are well-behaved in round $t$
i.e.  their true means are contained inside their confidence intervals.
\[W_t = \bigcap\limits_{a\in \{1,2,\dots,K\}} \left( (U_a(t) > p_a) \cap
(L_a(t) < p_a) \right)\]
Note that the event $W$ defined earlier is $W = \cup_{t\in\mathbb{N}} W_t$.

\cref{prop:prop1} gives a sufficient condition for stopping.
\begin{prop}
    Let $b_i \in [p_{\kappa_i}, p_{\kappa_i+1}]$. If $U_{u_t^i}-L_{l_t^i} > \epsilon$ and $W_t$ holds, then either
    $k=l_t^i$ or $k=u_t^i$ satisfies
    \[b_i \in \mathcal{I}_k(t) \text{ and }\tilde{\beta}_k(t) >
    \frac{\epsilon}{2},\]
    where we define $\tilde{\beta}_a(t) =
    \sqrt{\frac{\beta(t,\delta)}{2N_a(t)}}$
    \label{prop:prop1}
\end{prop}
\begin{proof}
Our $W_t$ condition is stronger than that required in the Proposition 1 in
\citet{kaufmann2013information}, and hence their proof applies.
\end{proof}
\cref{lem:concentrationinequality2} is another concentration result that will
be used in our sample complexity guarantee.
\begin{lemma}
    Let $T \ge 1$ be an integer, and $1 \le i \le (c-1)$ be any cluster
    boundary. Let $\delta>0, \gamma>0$ and $x \in ]0,1[$ be such that $p_a \neq x$.
    Then
    \[\sum\limits_{t=1}^{T} \P\left( a=u_i^t \lor a=l_i^t, 
        N_a(t) > \left\lceil \frac{\gamma}{\chernoff{p_a,x}} \right\rceil, 
        N_a(t) d(\hat{p}_a(t), x) \le \gamma \right) \le
    \frac{\exp(-\gamma)}{\chernoff{p_a, x}} \]
    \label{lem:concentrationinequality2}
\end{lemma}
We prove the following lemma, which states that the Chernoff information
increases as the second distribution moves away from the first.
%Another characterization of the Chernoff information is:
%\[\chernoff{x,y} = \inf\limits_{z^\ast \in [x,y]} \max (d(z^\ast,x),
%d(z^\ast,y))\]
%We shall use this version to prove the following property:
\begin{lemma}
    If $x < y < y'$ or $x > y > y'$, $\chernoff{x,y} \le \chernoff{x,y'}$
    \label{lem:chernoffincreasing}
\end{lemma}
\begin{proof}
    We shall prove the statement for the case $x < y < y'$. The proof for
    $x > y > y'$ is analogous. \\
    Let $z^\ast$ be the unique $z$ such that $d(z^\ast, x) = d(z^\ast, y)$.
    Since $z^\ast < y < y'$, $d(z^\ast, y') \ge d(z^\ast, y) = d(z^\ast, x)$. Hence, there
    exists $z^{\ast'} \ge z^\ast$ such that $\chernoff{x,y} = d(z^\ast, x) \le
    d(z^{\ast'},x) = d(z^{\ast'}, y') = \chernoff{x,y'}$.
\end{proof}
\begin{lemma}
    Let $x^\ast$ be the solution of the equation:
    \[x = \frac{1}{\gamma}\left( \log \frac{x^\alpha}{\eta} + \log \log
    \frac{x^\alpha}{\eta} \right)\]
    Then if $\gamma < 1$ and $\eta<1/e^e$, \[\frac{1}{\gamma} \log \left(
    \frac{1}{\eta \gamma^\alpha} \right) \le x^\ast \le \frac{C_0}{\gamma} \log
\left( \frac{1}{\eta \gamma^\alpha} \right)\]
    where $C_0$ is such that $C_0 \ge \left( 1+\frac{1}{e} \right) 
    \left( \alpha \log C_0 + 1 + \frac{\alpha}{e} \right)$.
    \label{lem:boundont1star}
\end{lemma}
\begin{proof}
    $x^\ast$ is upper bounded by any $x$ such that 
    $\frac{1}{\gamma}\left( \log \frac{x^\alpha}{\eta} + \log \log
    \frac{x^\alpha}{\eta} \right) \le x$.
    We look for $x^\ast$ of the form $\frac{C_0}{\gamma} \log \left(
    \frac{1}{\eta \gamma^\alpha} \right)$.
    \begin{align*}
    \frac{1}{\gamma}\left( \log \frac{x^\alpha}{\eta} + \log \log
    \frac{x^\alpha}{\eta} \right) &\le \frac{1}{\gamma}\left(
    1+\frac{1}{e} \right) \log \left( \frac{x^\alpha}{\eta} \right) \\ 
    &= \frac{1}{\gamma} \left( 1+\frac{1}{e} \right)\left( \alpha \log
        C_0 + \log \frac{1}{\eta \gamma^\alpha} + \alpha \log \log
        \frac{1}{\eta \gamma^\alpha} \right) \\
    &\le \frac{1}{\gamma} \left( 1+\frac{1}{e} \right)\left( \alpha \log
        C_0 + \left( 1+\frac{\alpha}{e} \right) \log \frac{1}{\eta
    \gamma^\alpha} \right) \\
    &\le \frac{1}{\gamma} \left( 1+\frac{1}{e} \right) \left( \alpha \log C_0 +
    1 + \frac{\alpha}{e}\right) \log \frac{1}{\eta \gamma^\alpha} 
\end{align*}
where the first and second inequalities hold because $\log x \le \frac{x}{e}$,
and the last inequality holds because $\frac{1}{\eta \gamma^\alpha} > e$.
Choosing $C_0$ such that 
\[C_0 \ge \left( 1+\frac{1}{e} \right)\left( \alpha \log C_0 + 1 +
\frac{\alpha}{e} \right)\]
gives us our upper bound.

To prove the lower bound, consider the series defined by 
\begin{align*}
    x_0 &= 1 \\
    x_{n+1} &= \frac{1}{\gamma} \left( \log \frac{x_n^\alpha}{\eta} + \log \log
    \frac{x_n^\alpha}{\eta} \right)
\end{align*}
First note that since $\gamma < 1$ and $\eta < 1/e^e$, the sequence is increasing.
Second, note that the sequence converges to $x^\ast$. Hence
\begin{align*}
    x^\ast \ge x_2 &= \frac{1}{\gamma}\left[ \log \left( \frac{1}{\eta
    \gamma^\alpha} \left( \log \frac{1}{\eta} + \log \log
    \frac{1}{\eta} \right) ^\alpha \right) + 
    \log \log \left( \frac{1}{\eta \gamma^\alpha} \left( \log \frac{1}{\eta} + \log \log
    \frac{1}{\eta} \right) ^\alpha \right) \right] \\
    &= \frac{1}{\gamma} \left[ \log \frac{1}{\eta \gamma^\alpha} +
    \alpha \log \left( \log \frac{1}{\eta} + \log \log \frac{1}{\eta} \right) +
    \log \log \frac{1}{\eta \gamma^\alpha} + 
    \alpha \log \log \left( \log \frac{1}{\eta} + \log \log \frac{1}{\eta}
    \right) \right] \\
    &\ge \frac{1}{\gamma} \log \frac{1}{\eta \gamma^\alpha}
\end{align*}
since $\eta < 1/e^e$.
\end{proof}

\begin{corr}
    Let $\gamma = \frac{1}{2\hstar}, \eta =
    \frac{\delta}{k_1 K}$. Then applying \cref{lem:boundont1star} gives
    \[2\hstar \log \left(\frac{k_1 K (2 \hstar)^\alpha}{\delta}\right)\le S_1^\ast
    \le 2C_0(\alpha) \hstar \log \left( \frac{k_1 K (2 \hstar)^\alpha}{\delta}
\right)\]
    \label{cor:boundont1star}
\end{corr}

\subsection{Proof of \cref{thm:samplecomplexity}}
\begin{theorem*}
    Let $b = (b_1,b_2,\dots,b_{c-1})$, where $b_i \in [p_{\kappa_i},
    p_{\kappa_i+1}]$. Let $\epsilon>0$. Let $\beta(t,\delta) = \log(\tfrac{k_1K
    t^\alpha}{\delta}) + \log \log (\tfrac{k_1K t^\alpha}{\delta})$ with $k_1 >
    1 + \tfrac{2e}{\alpha-1} + \tfrac{4e}{(\alpha-1)^2}$. Let $\tau$ be the
    random number of samples taken by $\lucbrank$ before termination. If $\alpha
    > 1$,
    \[\P\left(\tau \le 2 C_0(\alpha) H^\ast_{\epsilon,b} \log
    \left( \frac{k_1 K (2 H^\ast_{\epsilon, b})^\alpha}{\delta} \right)\right)
    \ge 1-\delta\]
    %Moreover, for $\alpha > 2$, we have the following upper bound on
    %$\EE{\tau}$:
    %\[\EE{\tau} \le 2 C_0(\alpha) \sum_{i=1}^{c-1}H^\ast_{\epsilon, c_i} \log
    %\left( \frac{k_1 K (c-1) (2H^\ast_{\epsilon, c_i})^\alpha}{\delta} \right) +
    %K_\alpha \]
    where
    %\[K_\alpha = \frac{\delta}{k_1 K (\alpha-2)} + \frac{\delta
    %2^{\alpha-1}}{k_1 (\alpha-2)} + \frac{2^{\alpha-1} \delta}{k_1}\left( 1+
    %\frac{1}{e} + \log 2^\alpha + \log (1+\alpha \log 2) \right)
    %\frac{1}{(\alpha-2)^2},\] and
    $C_0(\alpha)$ is such that $C_0(\alpha) \ge \left( 1+\frac{1}{e} \right) \left( \alpha \log (C_0(\alpha)) + 1 +
    \frac{\alpha}{e} \right)$.
\end{theorem*}
\begin{proof}
    The $\lucbrank$ algorithm proceeds in rounds. In a round, it samples the two
    arms on opposite sides of an active boundary whose confidence intervals overlap
    the most. A boundary is active as long as this overlap is less than
    $\epsilon$. Thus, the number of samples up to round $T$ is
\begin{align*}
    \text{\#samples}(T) &\le 2 \sum\limits_{t=1}^{T} \sum\limits_{i=1}^{c-1}
    \mathbbm{1}_{(U_{u_t^i}-L_{l_t^i} > \epsilon)} \\
    &= 2\sum\limits_{t=1}^{T} \sum\limits_{i=1}^{c-1}
    \mathbbm{1}_{(U_{u_t^i}-L_{l_t^i} > \epsilon)} (\mathbbm{1}_{W_t} +
    \mathbbm{1}_{W^c_t})\\
    &\le
    2\sum\limits_{t=1}^{T} \sum\limits_{i=1}^{c-1}
    \mathbbm{1}_{(U_{u_t^i}-L_{l_t^i} > \epsilon)} \mathbbm{1}_{W_t} +
    2\sum_{t=1}^T \sum\limits_{i=1}^{c-1} \mathbbm{1}_{W^c_t} \\
    &\le 2\sum\limits_{t=1}^{T} \sum\limits_{i=1}^{c-1} \sum_{a \in
        \{1,2,\dots,K\}} \mathbbm{1}_{(a=l_t^i) \lor (a=u_t^i)} \mathbbm{1}_{(b_i \in
        \mathcal{I}_a(t))} \mathbbm{1}_{(\tilde{\beta}_a(t) > 
        \frac{\epsilon}{2})} + 
        2\sum_{t=1}^T \sum\limits_{i=1}^{c-1}\mathbbm{1}_{W^c_t} \\
        &\text{(by \cref{prop:prop1})}
\end{align*}
    We now split the first sum into two depending on whether an arm $a$ belongs
    to the set $\mathcal{A}_\epsilon = \{a \in \{1,2,...,K\}:
    \Delta^\ast_{b} < \epsilon^2/2\}$. 
    \begin{align*}
        \text{\#samples}(T) &\le 2\sum\limits_{a \in \mathcal{A}_\epsilon}^{}
        \sum\limits_{t=1}^{T} \sum\limits_{i=1}^{c-1} \mathbbm{1}_{(a=l_t^i)
        \lor (a=u_t^i)} \mathbbm{1}_{\left( N_a(t) <
        \frac{\beta(t,\delta)}{\epsilon^2/2} \right)} +\\
        &\qquad 2\sum\limits_{a \in \mathcal{A}_\epsilon^c}^{} \sum\limits_{t=1}^{T} 
        \sum\limits_{i=1}^{c-1} \mathbbm{1}_{(a=l_t^i) \lor (a=u_t^i)} 
        \mathbbm{1}_{(b_i \in \mathcal{I}_a(t))} + 
        2\sum_{t=1}^T \sum\limits_{i=1}^{c-1} \mathbbm{1}_{W^c_t} \\
        &\le 2\sum\limits_{a \in \mathcal{A}_\epsilon}
        \frac{\beta(T,\delta)}{\epsilon^2/2} +
        2\sum\limits_{a \in \mathcal{A}_\epsilon^c}^{} \sum\limits_{t=1}^{T}
        \sum\limits_{i=1}^{c-1} \mathbbm{1}_{(a=l_t^i) \lor (a=u_t^i)} 
        \mathbbm{1}_{\left( N_a(t) \le \left\lceil
        \frac{\beta(T,\delta)}{\Delta^\ast_{b}(a)}\right\rceil \right)} +
        \\
        &\qquad \underbrace{2\sum\limits_{a \in \mathcal{A}_\epsilon^c}^{} \sum\limits_{t=1}^{T}
        \sum\limits_{i=1}^{c-1} \mathbbm{1}_{(a=l_t^i) \lor (a=u_t^i)} 
        \mathbbm{1}_{\left( N_a(t) > \left\lceil
        \frac{\beta(T,\delta)}{\Delta^\ast_{b}(a)}\right\rceil \right)} +
    2\sum_{t=1}^T \sum\limits_{i=1}^{c-1} \mathbbm{1}_{W^c_t}}_{R_T} \\
        &= 2\hstar\,\beta(T,\delta) + R_T 
    \end{align*}
    where 
    \[R_T = 
        2\sum\limits_{a \in \mathcal{A}_\epsilon^c}^{} \sum\limits_{t=1}^{T}
        \sum\limits_{i=1}^{c-1} \mathbbm{1}_{(a=l_t^i) \lor (a=u_t^i)} 
        \mathbbm{1}_{\left( N_a(t) > \left\lceil
        \frac{\beta(T,\delta)}{\Delta^\ast_{b}(a)}\right\rceil \right)} 
        \mathbbm{1}_{(b_i \in \mathcal{I}_a(t))} +
        2\sum_{t=1}^T \sum\limits_{i=1}^{c-1} \mathbbm{1}_{W^c_t}
    \]
    If we define $S_1^\ast = \min\{x: 2\hstar\, \beta(x, \delta) < x\}$, then we
    get that for $S > S_1^\ast$, the algorithm must have stopped before $S$
    samples on the event $(R_T=0)$. Denoting the total number of samples used by
    the algorithm by $\tau$, we have that, for any $S > S_1^\ast$,  $\P(\tau >
    S) \le \P(R_T \neq 0)$.
    \begin{align}
        &\P(\tau > S) \le \P(R_T \neq 0) \nonumber \\
        &\le \P\left( \exists\, a \in \mathcal{A}_\epsilon^c, t\le T, 1\le
        i\le (c-1): a = l_t^i \lor a = r_t^i, N_a(t) > \left\lceil
        \frac{\beta(T,\delta)}{\Delta^\ast_{b}(a)} \right\rceil, 
    b_i \in \mathcal{I}_a(t) \right) + \P(W^c)\nonumber \\
        %&\phantom{\le} \quad + \P\left( \bigcup\limits_{t=1}^{T} W_t^c \right)
        %\nonumber \\
        %
        &\le \P\left( \exists\, a \in \mathcal{A}_\epsilon^c, t\le T, 1\le
        i\le (c-1): a = l_t^i \lor a = r_t^i, N_a(t) > \left\lceil
        \frac{\beta(T,\delta)}{\chernoff{p_a,b_i}} \right\rceil, 
    b_i \in \mathcal{I}_a(t) \right) + \P(W^c) \label{eq:firstsecondterm} 
        %&\phantom{\le} \quad + \P\left( \bigcup\limits_{t=1}^{T} W_t^c
        %\right) \label{eq:firstsecondterm}
    \end{align}
    where the final inequality follows because $\Delta^\ast_{b}(a) \leq
    \chernoff{p_a,b_i}\,\forall\,1\le i\le c-1$ (by
    \cref{lem:chernoffincreasing}).

    Let us look at the first term:
    \begin{align*}
        &\P\left( \exists\, a \in \mathcal{A}_\epsilon^c, t\le T, 1\le
        i\le (c-1): a = l_t^i \lor a = r_t^i, N_a(t) > \left\lceil
        \frac{\beta(T,\delta)}{\chernoff{p_a,b_i}} \right\rceil, 
        b_i \in \mathcal{I}_a(t) \right) \\
        &\leq \sum\limits_{a \in \mathcal{A}_\epsilon^c}^{} 
        \sum\limits_{i=1}^{c-1} \sum\limits_{t=1}^{T} 
        \P\left(a = l_t^i \lor a = r_t^i, N_a(t) > \left\lceil 
        \frac{\beta(T,\delta)}{\chernoff{p_a,b_i}} \right\rceil,
        b_i \in \mathcal{I}_a(t) \right) \\
        &\leq \sum\limits_{a \in \mathcal{A}_\epsilon^c}^{} 
        \sum\limits_{i=1}^{c-1} \frac{\exp(-\beta(T,
        \delta))}{\chernoff{p_a,b_i}} \quad (\text{by }
        \cref{lem:concentrationinequality2})\\
        &\le (c-1) \exp(-\beta(T, \delta)) 
        \sum\limits_{a \in \mathcal{A}_\epsilon^c}^{}
        \frac{1}{\chernoff{p_a, b_i}} \\
        &\le (c-1)\hstar \exp(-\beta(T, \delta))\\
        &\le (c-1)\hstar \exp\left(-\beta\left(\tfrac{S_1^\ast}{c-1}, \delta\right)\right)
        \qquad (\text{if }\tau > S_1^\ast, T > \tfrac{S_1^\ast}{c-1})) \\
        &= (c-1) \hstar \cdot \frac{\delta(c-1)^\alpha}{k_1 K S_1^{\ast, \alpha}}
        \frac{1}{\log \frac{k_1 K S_1^{\ast, \alpha}}{\delta(c-1)^\alpha}} \\
        &\le (c-1)^{\alpha+1} \hstar \cdot \frac{\delta}{k_1 K \left( 2\hstar \log \left(\frac{k_1
        K (2\hstar)^\alpha}{\delta} \right) \right)^\alpha} \frac{1}{\log
        \frac{k_1 K S_1^{\ast, \alpha}}{\delta(c-1)^\alpha}} \qquad
        (\text{by the lower bound in }\cref{cor:boundont1star}) \\
        &\le \frac{\delta}{k_1}\cdot \left( \frac{c-1}{2} \right)^\alpha \qquad (\text{since }(c-1) \le K \text{ and
        }\alpha > 1)
    \end{align*}
    For the second term, note that 
    %The second term in \cref{eq:firstsecondterm} is bound using
    %\cref{lem:sumoferrorterms}. 
    %\[\P\left( \bigcup\limits_{t=1}^{T} W_t^c \right)) \le
    %2eK \sum\limits_{t=1}^{\infty} (\beta(t,\delta) \log t + 1)
    %e^{-\beta(t,\delta}\]
    \begin{alignat*}{2}
        \P(W^c)
        &\le 2eK\sum\limits_{t=1}^{\infty} (\beta(t,\delta) \log t + 1)
        \exp(-\beta(t,\delta)) \quad &(\text{by
        \cref{lem:concentrationinequality1}}) \\
        &\le \frac{\delta}{k_1} \left( \frac{4e}{(\alpha-1)^2} +
        \frac{2e}{(\alpha-1)} \right)  &(\text{by \cref{lem:sumoferrorterms}})
       % (c-1) 2eK \frac{\delta}{k_1 K}\left( \frac{2}{(\alpha-1)^2} +
    %\frac{1}{(\alpha-1)} \right) = \frac{\delta}{k_1}\left(
%\frac{4e}{(\alpha-1)^2} + \frac{2e}{(\alpha-1)} \right)
\end{alignat*}

    Substituting in \cref{eq:firstsecondterm}, we get that for $S > S_1^\ast$,
    \[\P(\tau > S) \le 
        \frac{\delta}{k_1}\left( \left(\frac{c-1}{2}\right)^\alpha
        + \frac{4e}{(\alpha-1)^2} +\frac{2e}{(\alpha-1)} \right) \le \delta\]
    by the choice of $k_1$.
\end{proof}

\subsection{Lower Bound}
The proof uses standard change of measure arguments used to prove lower
bounds. For bandit problems, this is succinctly expressed through Lemma
$1$ in \citet{kaufmann2015complexity} that we restate here for completeness.
\begin{lemma} \label{lem:lemma1} 
    Let $p$ and $p'$ be two bandit models with $K$ arms such that for all $a$,
    the distributions $p_a$ and $p_a'$ are mutually absolutely continuous. Let
    $\sigma$ be a stopping time with respect to $(\mathcal{F}_t)$ and let
    $A \in \mathcal{F}_\sigma$. Then
\[
    \sum\limits_{a=1}^{K} \mathbb{E}_p[N_a(\sigma)] \KL{p_a, p'_a} \ge
    d(\mathbb{P}_p(A), \mathbb{P}_{p'}(A))
\]
where $d(x,y) = x \log (x / y) + (1-x) \log ( (1-x) / (1-y))$.
\end{lemma}

\subsubsection{Proof of \cref{thm:lowerbound}}
    Consider any arm $a$. By \cref{ass:assumption}, there exists alternative
    model $p'$ such that: 
    \begin{align*}
        \KL{p_a, p_{\kappa_{g(a)}+1}} < \KL{p_a,p'_a} < \KL{p_a,
        p_{\kappa_{g(a)}+1}} + \alpha \text{ and } p'_a < p_{\kappa_{g(a)}+1}
    \end{align*}
    Note that in the model $p'$, arm $a$ no longer belongs to the cluster
    $g(a)$. Let $\hat{M}_{g(a)}$ be the set of arms returned by an algorithm in
    the $g(a)^\text{th}$ cluster. If we define the event $A = \{a \in
    \hat{M}_{g(a)}\} \in \mathcal{F}_\tau$, then by definition, for any
    $\delta$-PAC algorithm, $\mathbb{P}_p(A) \ge 1-\delta$ and
    $\mathbb{P}_{p'}(A) \le \delta$. Letting $N_a(\tau)$ denote the number of
    pulls of arm $a$ by time $\tau$, we have by \cref{lem:lemma1} and the monotonicity
    of $d(x,y)$ that 
    \[
        \KL{p_a, p'_a} \mathbb{E}_p[N_a(\tau)] \ge d(1-\delta, \delta) \ge \log
        (\tfrac{1}{2.4 \delta})
    \]
    where we use the property that for $x \in [0,1]$, $d(x,1-x) \ge \log
    \tfrac{1}{2.4x}$. This gives us that
    \[
        \mathbb{E}_p[N_a(\tau)] \ge \frac{1}{\KL{p_a, p_{\kappa_{g(a)}+1}} +
        \alpha} \log \left( \frac{1}{2.4\delta} \right)
    \]
    Letting $\alpha \rightarrow 0$, we get 
    \begin{equation}
        \mathbb{E}_p[N_a(\tau)] \ge \frac{1}{\KL{p_a, p_{\kappa_{g(a)}+1}}} 
        \log \left( \frac{1}{2.4\delta} \right)
        \label{eq:lbproofeqn1}
    \end{equation}
    Similarly, by considering an alternative model $p''$ such that 
    \[
        \KL{p_a, p_{\kappa_{g(a)-1}}} < \KL{p_a, p''_a} < \KL{p_a,
        p_{\kappa_{g(a)-1}}} + \alpha \text{ and } p''_a >
        p_{\kappa_{g(a)-1}}
    \]
    we get 
    \begin{equation}
        \mathbb{E}_p[N_a(\tau)] \ge \frac{1}{\KL{p_a, p_{\kappa_{g(a)-1}}}}
        \log \left( \frac{1}{2.4\delta} \right)
        \label{eq:lbproofeqn2}
    \end{equation}
    From \cref{eq:lbproofeqn1}, \cref{eq:lbproofeqn2}, and the definition of
    $\Delta^{\text{KL}}_\kappa(a)$ in \cref{eq:kldistances}, we get that 
    \[
        \mathbb{E}_p[N_a(\tau)] \ge \frac{1}{\Delta^{\text{KL}}_\kappa(a)}
        \log \left( \frac{1}{2.4\delta} \right)
    \]
    Summing over all the arms yields the required bound for 
    $\mathbb{E}_p[\tau] = \sum_{a=1}^{K} \mathbb{E}_p[N_a(\tau)]$.

\end{document}